\pdfoutput=1
\documentclass{article}

\usepackage[accepted]{aistats2026}
\usepackage{algorithm}
\usepackage{algorithmic}
\usepackage{amsmath}
\usepackage{amssymb}
\usepackage{mathtools}
\usepackage{amsthm}
\usepackage{bm}
\usepackage{cite}
\usepackage{url}
\usepackage{color}
\usepackage{booktabs}
\usepackage{multirow}
\usepackage{graphicx}
\usepackage{enumitem}
\usepackage[colorlinks,linkcolor=blue,citecolor=blue,urlcolor=blue]{hyperref}

\newtheorem{dfn}{Definition}[section]

\newtheorem{lem}{Lemma}[section]

\newcommand{\UPDATE}[1]{{#1}}

\begin{document}

\runningtitle{Lipschitz Multiscale Deep Equilibrium Models}

\twocolumn[

\aistatstitle{Lipschitz Multiscale Deep Equilibrium Models:\\ A Theoretically Guaranteed and Accelerated Approach}

\aistatsauthor{ Naoki Sato \And Hideaki Iiduka}

\aistatsaddress{ Meiji University \And Meiji University} ]

\begin{abstract}
Deep equilibrium models (DEQs) achieve infinitely deep network representations without stacking layers by exploring fixed points of layer transformations in neural networks.
Such models constitute an innovative approach that achieves performance comparable to state-of-the-art methods in many large-scale numerical experiments, despite requiring significantly less memory.
However, DEQs face the challenge of requiring vastly more computational time for training and inference than conventional methods, as they repeatedly perform fixed-point iterations with no convergence guarantee upon each input.
Therefore, this study explored an approach to improve fixed-point convergence and consequently reduce computational time by restructuring the model architecture to guarantee fixed-point convergence.
Our proposed approach for image classification, Lipschitz multiscale DEQ, has theoretically guaranteed fixed-point convergence for both forward and backward passes by hyperparameter adjustment, achieving up to a 4.75$\times$ speed-up in numerical experiments on CIFAR-10 at the cost of a minor drop in accuracy.
\end{abstract}

\section{Introduction}
The past decade has seen remarkable progress in explicit deep learning, largely driven by the development of increasingly deep neural networks. Architectures like ResNet \cite{He2016Dee} and Transformers \cite{Vaswani2017Att} have demonstrated that performance on a wide range of tasks consistently improves with model depth. This ``deeper is better" paradigm has led to models with tens or even hundreds of explicit layers. However, this pursuit of depth comes at a significant cost: (1) substantial memory requirements for storing activations of all intermediate layers for backpropagation, and (2) diminishing returns in performance, where simply stacking more layers can lead to vanishing gradients and optimization difficulties. These challenges have motivated a search for alternative ways to conceptualize and leverage depth in neural networks.

Deep equilibrium models (DEQs) \cite{Bai2019Dee, Bai2020Mul} offer a radical departure from the explicit, layer-stacking approach. Instead of defining a finite sequence of transformations, a DEQ layer conceptualizes its forward pass as the solution to a fixed-point equation. By iteratively applying a single, weight-tied nonlinear transformation until it reaches an equilibrium point, a DEQ effectively creates a neural network of infinite depth. This implicit formulation carries a profound advantage: a drastically reduced memory footprint for training. Unlike explicit $N$-layer models that require $\mathcal{O}(N)$ memory to store intermediate activations for backpropagation, a DEQ's memory requirement is $\mathcal{O}(1)$ and remains constant, effectively decoupling model depth from memory cost. This allows for theoretically infinite-depth representations without the memory constraints that plague traditional deep networks. Impressively, both the DEQ-Transformer \cite{Bai2019Dee} for text tasks and multiscale DEQs (MDEQs) \cite{Bai2020Mul} for computer vision tasks achieve this excellent memory efficiency while delivering performance comparable to state-of-the-art models.

Despite the promise of DEQs, their practical adoption is severely hindered by a singular, critical issue: prohibitive training and inference times. Empirical results show that these models can be 2.5 to 3 times slower than their explicit counterparts \cite{Bai2019Dee, Bai2020Mul}. We argue that this computational burden is a direct symptom of a deeper theoretical problem: the lack of a guaranteed unique solution to the fixed-point problem they aim to solve. For each new input, a DEQ must solve the following fixed-point problem using some fixed-point solver.
\begin{align}\label{eq:forward}
\text{Find } \bm{z} \in \text{Fix}(f_{\bm{\theta}}) := \left\{ \bm{z} \in \mathbb{R}^d \colon f_{\bm{\theta}}(\bm{z}; \bm{x}) = \bm{z} \right\},
\end{align}
where $\bm{z}$ is hidden states, $\bm{x}$ is input, $\bm{\theta}$ is weight parameters, and $f_{\bm{\theta}} \colon \mathbb{R}^d \to \mathbb{R}^d$ is the single, repeated transformation. 
From the Banach fixed-point theorem \cite{Banach1922Leb, Kreyszig1978Int}, a unique solution to this problem is guaranteed to exist only if the mapping $f_{\bm{\theta}}$ is a contraction mapping. 
Here, a mapping $f \colon \mathbb{R}^d \to \mathbb{R}^d$ is said to be a contraction mapping if there exists $r \in (0,1)$ such that, for all $\bm{z}_1, \bm{z}_2 \in \mathbb{R}^d$, $\| f(\bm{z}_1) - f(\bm{z}_2) \| \leq r \| \bm{z}_1 - \bm{z}_2 \|$. In other words, being a contraction mapping is equivalent to being an $L$-Lipschitz mapping for $L<1$. 
It is known that if a mapping $f$ is a contraction mapping, then the sequence generated by the Banach fixed-point approximation method, $\bm{z}_{t+1} := f(\bm{z}_t)$, converges to a unique fixed point of $f$.
However, $f_{\bm{\theta}}$ does not generally satisfy this contractility, because $f_{\bm{\theta}}$ is a complex combination of convolutions, regularization, nonlinear activations, etc. As a result, fixed-point solvers are forced to perform iterative processes with no convergence guarantee, resulting in unpredictable and often explosive computational costs for both training and inference.

In the present study, our aim was to redesign the model architecture (i.e., $f_{\bm{\theta}}$) so that there exists a unique solution to the fixed-point problem \eqref{eq:forward}, thereby providing theoretical guarantees for fixed-point convergence, accelerating convergence, and reducing the enormous computational time required for DEQ training and inference. Our approach is rooted in directly controlling the Lipschitz constant of the fixed-point mapping. While the principles of our approach could be generalized, the model structures of DEQ variants like the DEQ-Transformer and MDEQ are fundamentally different. Therefore, we focus here exclusively on MDEQ for computer vision tasks to rigorously validate the effectiveness of our techniques. Our contribution can be summarized as follows:
\begin{enumerate}[leftmargin=*,itemsep=0pt, topsep=0pt]
\item We propose a new variant of MDEQ, Lipschitz MDEQ, which incorporates several modifications to the MDEQ architecture. In contrast to MDEQ, the Lipschitz constant of the fixed-point mapping appearing in Lipschitz MDEQ is composed of several hyperparameters that can be determined by the user. \textbf{(Section \ref{sec:3.3})}

\item We demonstrate a combination of hyperparameters that ensures the fixed-point mappings appearing in both the forward pass and backward pass of Lipschitz MDEQ is a contraction mapping. Consequently, Lipschitz MDEQ is a member of the DEQ family that theoretically guarantees fixed-point convergence. \textbf{(Section \ref{sec:3.4})}

\item Our numerical experiments demonstrate that Lipschitz MDEQ significantly improves fixed-point convergence, reduces MDEQ computation time by up to 4.75-fold with theoretical guarantees, and enables trade-offs between accuracy and speed by adjusting the Lipschitz constant. \textbf{(Section \ref{sec:exp})}
\end{enumerate}

\section{Background \& Related Work}
For related work on other theory, applications, adversarial robustness, the implicit layers family, and Lipschitz networks, see Appendix \ref{sec:other}.
\paragraph{Weight-tied Models}
DEQs originate from the breakthroughs of models such as TrellisNet \cite{Bai2019Tre} and the Universal Transformer \cite{Dehghani2019Uni}, which have shared weights across all layers.
Although these models achieve constant memory in the sense that the number of parameters does not increase with the number of layers, they still have the drawback of the computational cost being proportional to the number of layers, due to backpropagation.
Bai et al.\ observed experimentally that the output of the intermediate layer of these models, i.e., the hidden state $\bm{z}$, converges to a fixed point of the transformation $f_{\bm{\theta}}$ as the number of layers increases \cite{Bai2019Dee}. They proposed the DEQ approach, shifting to directly finding the fixed point $\bm{z}^\star$ of $f_{\bm{\theta}}$ rather than adding layers to obtain it. This approach was incorporated into the transformer and demonstrated extremely powerful performance in large-scale sequence modeling tasks \cite{Bai2019Dee}. Soon after, MDEQ, incorporating a multiscale mechanism to make it applicable to computer vision tasks, was proposed and shown to demonstrate extremely good performance in both image classification and semantic segmentation tasks \cite{Bai2020Mul}.

\paragraph{Deep Equilibrium Models}
A DEQ receives an input and solves the fixed-point problem \eqref{eq:forward} using some fixed-point solver in the forward pass. The number of iterations by the solver corresponds to the number of layers in the DEQ.
Once the fixed point $\bm{z}^\star$ is found, it is used to define an empirical loss function $\ell_{\bm{z}^\star}(\bm{\theta})$, and the loss minimization procedure proceeds using a standard optimizer such as stochastic gradient descent or Adam. The required gradient can be obtained in the following form by applying the chain rule and implicit function theorem \cite{Krantz2012The}:
\begin{align}\label{eq:back}
\frac{\partial \ell_{\bm{z}^\star}(\bm{\theta})}{\partial \bm{\theta}} = \underbrace{\frac{\partial \ell_{\bm{z}^\star}(\bm{\theta})}{\partial \bm{z}^\star} (I - J_{f_{\bm{\theta}}}(\bm{z}^\star))^{-1}}_{=: A} \frac{\partial f_{\bm{\theta}}(\bm{z}^\star; \bm{x})}{\partial \bm{\theta}},
\end{align}
where $J_{f_{\bm{\theta}}}$ is the Jacobian matrix of $f_{\bm{\theta}}$ at $\bm{z}^\star$. To avoid the expensive computation of the inverse of the Jacobian matrix, the term $A$ is calculated in the backward pass by solving the following equation:
\begin{align}\label{eq:linear}
\bm{\mathrm{x}} \left( I - J_{f_{\bm{\theta}}}(\bm{z}^\star) \right) - \frac{\partial \ell_{\bm{z}^\star}(\bm{\theta})}{\partial \bm{z}^\star} = \bm{0}. 
\end{align}
The first term can be efficiently computed for any $\bm{\mathrm{x}}$ via autograd packages such as PyTorch \cite{Paszke2019PyT}. The linear equation \eqref{eq:linear} is equivalent to the problem of finding a fixed point that satisfies 
\begin{align}\label{eq:back-fix}
\bm{\mathrm{x}} = T(\bm{\mathrm{x}}) \ \text{under}\  T(\bm{\mathrm{x}}) := \bm{\mathrm{x}} J_{f_{\bm{\theta}}}(\bm{z}^\star) + \frac{\partial \ell_{\bm{z}^\star}(\bm{\theta})}{\partial \bm{z}^\star}.
\end{align}
In other words, DEQs solve two different fixed-point problems in the backward and forward passes.
In both fixed-point problems, the initial research used the Broyden's method \cite{Broyden1965ACl} as the solver, but recent research has focused on the Anderson acceleration \cite{Anderson1965Ite}, which is provably equivalent to a multi-secant quasi-Newton method \cite{Fang2009Two}, has become mainstream. In addition, some DEQ variants \cite{Pabbaraju2021Est, Gabor2024Pos} use the Peaceman-Rachford method \cite{Peaceman1955The} or the Banach fixed-point approximation method \cite{Banach1922Leb} under different formulations. Recently, a RevDEQs that enables the use of backpropagation in the backward pass were proposed that use a reversible fixed-point solver \cite{McCallum2025Rev}.

\paragraph{Theory}
Generally, for the fixed-point problems solved by DEQs, the existence and uniqueness of fixed points are not guaranteed. Several prior studies have proposed variants of DEQs to ensure these properties. Winston et al.\ introduced Monotone Operator Equilibrium Networks (monDEQ), ensuring the existence of a fixed point by designing the layer as a monotone operator \cite{Winston2020Mon}. Revay et al.\ achieved uniqueness by enforcing contraction under a learnable non-Euclidean norm, which relaxes constraints on the linear layers \cite{Revay2020Lip}. More recently, Sittoni et al.\ proposed Subhomogeneous DEQ (SubDEQ), guaranteeing the existence and uniqueness of a bounded solution by imposing subhomogeneity on the mapping \cite{Sittoni2024Sub}. Concurrently, Gabor et al.\ introduced Positive Concave DEQ (pcDEQ), which ensures a unique fixed point by decomposing the function into contractive and non-expansive parts \cite{Gabor2024Pos}.

\paragraph{Acceleration Techniques}
DEQs are also known for their extremely low training and inference speeds.
To address this issue, several prior studies have successfully reduced the computational cost of the backward pass \eqref{eq:back} by utilizing approximate gradients \cite{Geng2021OnT, Ramzi2022SHI, Fung2022JFB, Nguyen2022Eff}. In particular, Fung et al.'s Jacobian-free backpropagation (JFB) achieves acceleration by replacing $(I- J_{f_{\bm{\theta}}}(\bm{z}^\star))^{-1}$ with the identity matrix, whereas Geng et al.\ achieved acceleration by using the phantom gradient of a surrogate loss $\frac{1}{2}\| f_{\bm{\theta}}(\bm{z}^\star) - \bm{z}^\star \|^2$, which is easier to compute \cite{Geng2021OnT}.
Additionally, there are techniques that utilize differences in convergence rates between dimensions within the hidden dimensions of the DEQ to omit computations \cite{Wang2024Del}, as well as methods that introduce learnable neural fixed-point solvers to streamline fixed-point iterations \cite{Bai2022Neu}. While slightly different in nature, Gurumurthy et al.'s proposed method significantly reduces computation time by simultaneously solving fixed-point iterations and input optimization for specific tasks \cite{Gurumurthy2021Joi}. From the perspective of accelerating DEQs, the prior study closest to our approach is that proposing Jacobian regularization \cite{Bai2021Sta}. While Jacobian regularization successfully stabilizes training by the employment of Jacobian norm $\| J_{f_{\bm{\theta}}}(\bm{z^\star}) \|_{\rm{F}}$ as a penalty term for the backward pass, our approach achieves acceleration grounded in the theory of both the forward and backward passes by constraining the supremum of $\| J_{f_{\bm{\theta}}}(\bm{z}) \|_2$ to be less than 1. Furthermore, although the Jacobian regularization approach does not have theoretical convergence guarantees, it has the practical advantage of incurring minimal performance degradation. Note that Jacobian regularization stabilizes fixed-point convergence to achieve high accuracy with fewer iterations, but it does not guarantee or accelerate fixed-point convergence.

\begin{figure*}[htbp]
\begin{minipage}[t]{1\linewidth}
\centering
\includegraphics[width=1\linewidth]{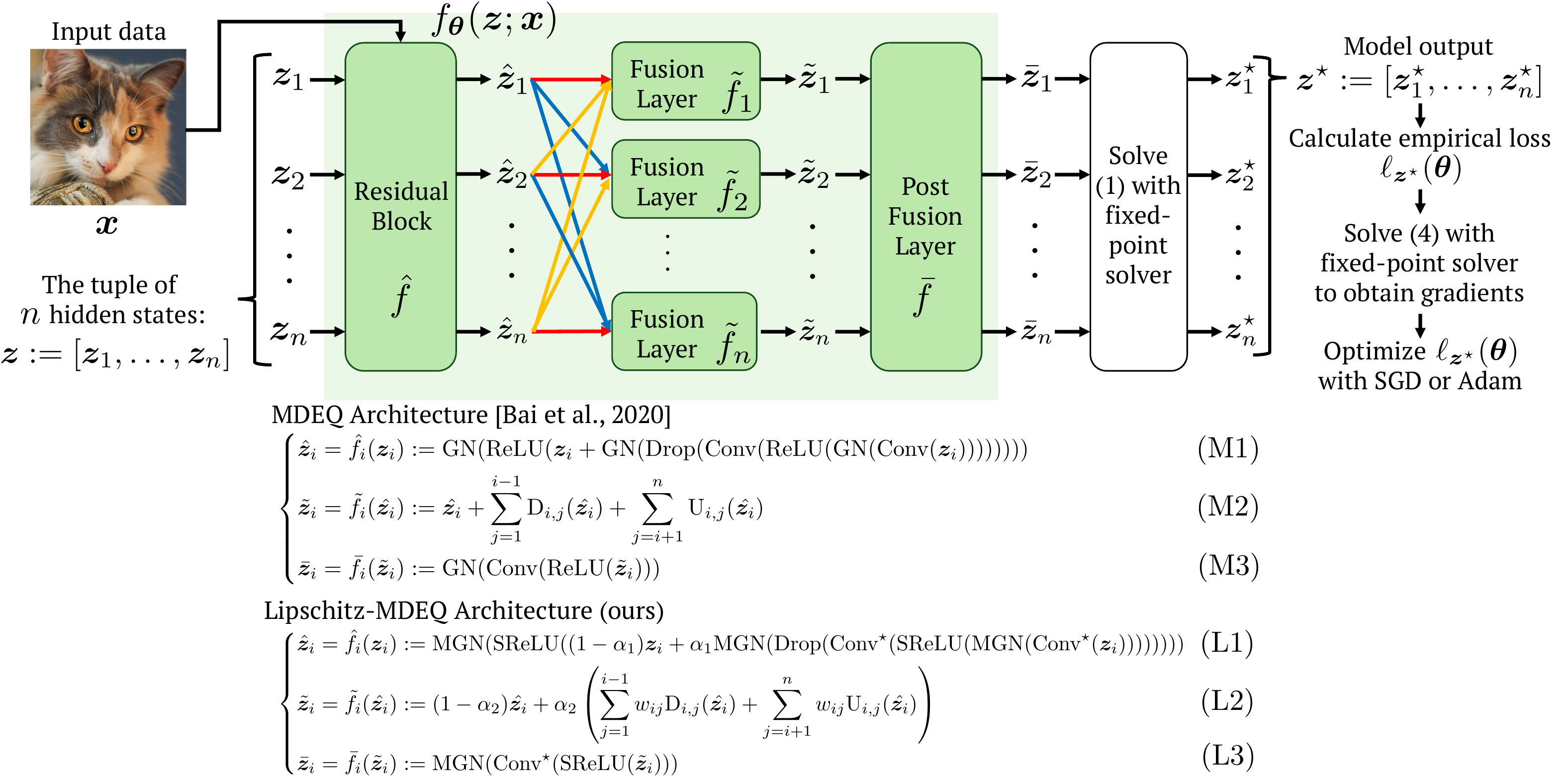}%
\end{minipage} 
\vspace*{-20pt}
\caption{Architecture of MDEQ and Lipschitz MDEQ. Lipschitz MDEQ inherits the general structure of MDEQ, with each operation modified to prevent the Lipschitz constant from becoming too large. Note that the layout and notation of the figure are taken from the prior work \cite[Figure 1]{Bai2020Mul} with only minor modifications made for clarity in explaining our results. The sample image is from the AFHQ dataset \cite{Choi2020Sta}.}
\label{fig:mdeq}
\vspace*{-8pt}
\end{figure*}

\section{Main Results}\label{sec:main}
\subsection{Preliminaries}
\vspace*{-5pt}
\paragraph{Notation} Let $\mathbb{N}$ be the set of non-negative integers. For $n \in \mathbb{N}$, define $[n] := \{ 1,2, \ldots, n\}$. Let $\mathbb{R}^d$ be a $d$-dimensional Euclidean space with inner product $\langle \cdot, \cdot \rangle$, which induces the norm $\| \cdot \|$. Let $f \circ g$ be the composite function of $f$ and $g$, i.e., $(f \circ g)(\cdot) := f(g(\cdot))$. Here, $f^k$ will denote applying function $f$ $k$ times, with $f^0$ denoting the identity mapping. We use $\odot$ to denote the element-wise product, also known as the Hadamard product.
\begin{dfn}[$L$-Lipschitz mapping]\label{dfn:lipschitz}
Given a mapping $f \colon \mathbb{R}^d \to \mathbb{R}^d$, if $\| f(\bm{x}) - f(\bm{y}) \| \leq L_f \| \bm{x} - \bm{y} \|$ holds for any $\bm{x}, \bm{y} \in \mathbb{R}^d$, then $f$ is said to be an $L_f$-Lipschitz mapping, and $L_f$ is called the Lipschitz constant of $f$.
\end{dfn}

As in Definition \ref{dfn:lipschitz}, we use $L_f$ to represent the Lipschitz constant of a mapping $f$. For specific operators, we use simplified versions such as $L_{\text{Conv}}$ and $L_{\text{ReLU}}$. The following lemma is frequently used to derive the Lipschitz constant for the forward pass of MDEQ, and is here used for the proposed Lipschitz MDEQ.

\begin{lem}\label{lem:01}
Let the mappings $f \colon \mathbb{R}^d \to \mathbb{R}^d$ and $g \colon \mathbb{R}^d \to \mathbb{R}^d$ be $L_f$-Lipschitz and $L_g$-Lipschitz mappings, respectively. Then, the mapping $f + g$ is an $(L_f + L_g)$-Lipschitz mapping. Also, the mapping $f \circ g$ is an $L_fL_g$-Lipschitz mapping.
\end{lem}
First, we clarify how the fixed-point map $f_{\bm{\theta}}$ appearing in MDEQ is defined. In MDEQ, hidden states are defined for each resolution, processed in parallel, and influence each other. Let the number of resolutions be $n \in \mathbb{N}$, and let $\bm{z}_1$ and $\bm{z}_n$ denote the hidden states at the highest and lowest resolutions, respectively. The input image is injected only into $\bm{z}_1$, and the other hidden states receive input only indirectly. However, note that since we are considering the Lipschitz constant of $f_{\bm{\theta}}$ with respect to the hidden state $\bm{z} := [\bm{z}_1, \ldots, \bm{z}_n]$, the processing of the input can be ignored. $f_{\bm{\theta}}$ consists of three blocks: a residual block $\hat{f} \colon \bm{z} \mapsto \hat{\bm{z}}$ (Eq. (M1)), a fusion layer $\tilde{f} \colon \hat{\bm{z}} \mapsto \tilde{\bm{z}}$ (Eq. (M2)), and a post-fusion layer $\bar{f} \colon \tilde{\bm{z}} \mapsto \bar{\bm{z}}$ (Eq. (M3)) (see Figure 1), noting that $f_{\bm{\theta}} = \bar{f} \circ \tilde{f} \circ \hat{f}$. The operations constituting these will be detailed in Section \ref{sec:3.3} along with our modifications.

\subsection{Lipschitz MDEQ Architecture}\label{sec:3.3}
We will detail each operation appearing in MDEQ and derive its Lipschitz constant. Then we will modify each operation so that the Lipschitz constant of the mapping $f_\theta$ is less than 1, thereby constructing the Lipschitz MDEQ.
\paragraph{Normalization}
MDEQ employs Group Normalization (GN) \cite{Wu2018Gro}, which normalizes features within groups of channels. For an input group of features $\bm{z}$, the standard GN operation is defined as 
\begin{align*}
\text{GN}(\bm{z})_i = \gamma \frac{z_i - \mu_{\bm{z}}}{\sqrt{\sigma^2_{\bm{z}} + \epsilon}} + \beta,
\end{align*}
where $\mu_{\bm{z}} \in \mathbb{R}$ and $\sigma^2_{\bm{z}} \in \mathbb{R}$ are the mean and variance, and $\gamma, \beta \in \mathbb{R}$ are learnable affine parameters. The Lipschitz constant of this operation is upper-bounded as $L_{\text{GN}} \leq \frac{|\gamma|}{\sqrt{\epsilon}}$ (see Lemma \ref{lem:gn}). Since $\epsilon$ is a very small constant (e.g., $10^{-5}$) to prevent division by zero, the Lipschitz constant can become excessively large in the worst-case scenario where the variance approaches zero. This potential for an explosive Lipschitz constant poses a significant challenge to the stability of the fixed-point iteration.
To mitigate this issue, our Lipschitz MDEQ adopts a simpler normalization scheme, Mean-Only Group Normalization (MGN). This operation only centers the features by subtracting the mean, omitting the division by the standard deviation:
\begin{align*}
\text{MGN}(\bm{z})_i = \gamma (z_i - \mu_{\bm{z}}) + \beta.%
\end{align*}
This seemingly small modification has a profound impact on the Lipschitz constant. The upper bound for MGN is given by $L_{\text{MGN}} \le |\gamma|$ (see Lemma \ref{lem:mgn}). By removing the problematic $1/\sqrt{\epsilon}$ term, the Lipschitz constant becomes directly dependent only on the learnable parameter $\gamma$, making it significantly smaller and far more controllable. We treat the upper bound $\bar{\gamma}$ on the absolute value of $\gamma$ as another adjustable hyperparameter, enabling precise control over the Lipschitz contribution from the normalization layer. 
In the implementation, the use of learnable parameters $\gamma$ and $\beta$ is optional. In the MDEQ implementation, the GN appearing in the residual block and fusion layer uses this option, and Lipschitz MDEQ adopts this approach.
When these are not used, $\gamma=1$ and $\beta=0$ are set and therefore $L_{\text{MGN}}$ is always 1. When using them, $L_{\text{MGN}}$ can be controlled by arbitrarily determining $\bar{\gamma}$.

\paragraph{Activation Function}
MDEQ employs the ReLU function as its activation function. The ReLU function operates on each element of the input vector $\bm{z}$ and is defined as $\text{ReLU}(\bm{z})_i := \max\{ 0, z_i\}$. The Lipschitz constant for this operation is clearly 1 (see Lemma \ref{lem:relu}).
To contribute to keeping the overall Lipschitz constant small, our Lipschitz MDEQ adopts the Scaled-ReLU (SReLU) function \cite{Ramy2020OnP}. SReLU is defined as $\text{SReLU}(\bm{z})_i := \max\{ 0, az_i \}$, where $a \in (0,1]$ is a hyperparameter and the Lipschitz constant becomes $L_{\text{SReLU}}=a$ (see Lemma \ref{lem:srelu}).

\paragraph{Dropout}
The dropout \cite{Srivastava2014Dro} is an important technique for suppressing overfitting and improving generalization performance.
MDEQ employs variational dropout \cite{Gal2016ATh} defined by $\text{Drop}(\bm{z}) := \frac{1}{1-p}\bm{m} \odot \bm{z}$, where $p \in (0,1)$ is the dropout rate and $\bm{m} \sim \text{Bernoulli}(p)$ is the mask. The Lipschitz constant for this operation is $L_{\text{Drop}} = 1/(1-p)$ (see Lemma \ref{lem:drop}). In practice, dropout rate $p$ is typically a small value ranging from 0 to 0.3, so the Lipschitz constant for the dropout operation does not become excessively large. Therefore, our Lipschitz MDEQ also adopts variational dropout.

\paragraph{Convolution}
A standard convolution layer performs an affine transformation on the input feature map $\bm{z}$. Conceptually, this operation can be expressed as $\text{Conv}(\bm{z}) = W\bm{z} + \bm{b}$, where $W$ is a linear operator derived from the convolutional kernel and $\bm{b}$ is a bias term. 
The Lipschitz constant for this operation is $L_{\text{Conv}} = \| W \|_2$ (see Lemma \ref{lem:conv}).
In our Lipschitz MDEQ, we directly control this spectral norm to regulate the Lipschitz constant of the convolutional layers. Specifically, we enforce a constraint $\| W \|_2 \leq c$ by reprojecting the weights after each optimization step, where the upper bound $c > 0$ is a tunable hyperparameter we refer to as the target norm. We denote constrained $\text{Conv}$ as $\text{Conv}^\star$, and therefore $L_{\text{Conv}^\star} = c$.

\paragraph{Residual Connection}
In the residual block of MDEQ, a residual connection is observed where the inputs are simply added together, such as $h(x) = x + g(x)$. This is an extremely effective mechanism introduced in \cite{He2016Dee}. However, since the Lipschitz constant of the identity mapping is 1, $L_h = 1 + L_g$ according to Lemma \ref{lem:01}, contributing to an overall increase in the Lipschitz constant. The simplest approach to mitigate this is to use a convex combination for $\alpha \in (0,1)$, such as $h'(x) = (1-\alpha)x + \alpha g(x)$. Then, we have $L_{h'} = (1-\alpha) + \alpha L_g$, achieving a smaller Lipschitz constant than the simple residual connection. Therefore, in Lipschitz MDEQ, all residual connections are made convex combinations, and the coefficient parameters are treated as hyperparameters. The convex combination parameter in the residual block is denoted by $\alpha_1$ (Eq. (L1)).
It should be noted that this convex combinations structure has been used in highway networks \cite{Srivastava2015Hig, Srivastava2015Tra} and has been widely adopted as the fundamental form of gate mechanisms in gated recurrent units (GRUs) \cite{Chung2014Emp} and recent gated linear RNNs \cite{Qin2023Hie, Feng2024Wer, Gu2024Mam}.

\paragraph{Downsample}
In the fusion layer of MDEQ, when passing states from high resolution to low resolution, a convolution operation is applied that is proportional to the difference in resolution (the blue arrow in Figure \ref{fig:mdeq}). Specifically, if $h_1(\bm{z}) := \text{GN}(\text{CONV}(\bm{z}))$ and $h_2(\bm{z}) := \text{ReLU}(\text{GN}(\text{CONV}(\bm{z})))$, the operation $\text{D}_{i,j}$ is defined as $\text{D}_{i,j} := h_1 \circ h_2^{i-j-1}$. Therefore, the Lipschitz constant for this operation is $L_{\text{Down}}^{i,j} = L_{h_1}L_{h_2}^{i-j-1}$, which depends on the Lipschitz constants of the activation function, normalization, and convolution operations. Our Lipschitz MDEQ implements the MDEQ downsampling operation without altering its form. Therefore, in Lipschitz MDEQ, the Lipschitz constant for this operation is $L_{\text{Down}}^{i,j} = L_{\text{MGN}}L_{\text{Conv}^\star}(L_{\text{SReLU}}L_{\text{MGN}}L_{\text{Conv}^\star})^{i-j-1}$.

\paragraph{Upsample}
In the fusion layer of MDEQ, when passing states from low resolution to high resolution, missing dimensions are filled using nearest-neighbor interpolation (the yellow arrow in Figure \ref{fig:mdeq}). This operation $\text{U}_{i,j}$ takes a low-resolution input $\bm{z}_j \in \mathbb{R}^{H \times W}$ and outputs high-resolution $\bm{z}_{j \to i} \in \mathbb{R}^{sH \times tW}$, where $s, t \in \mathbb{N}$ are the vertical and horizontal scale factors, respectively. In MDEQ, both scale factors are defined as $2^{\text{level\_diff}}$, where $\text{level\_diff}$ is the difference between the indices of the hidden states; for example, $\text{level\_diff}$ between $\bm{z}_1$ and $\bm{z}_3$ is 2. The Lipschitz constant for this operation is $L_{\text{up}}=\sqrt{st}$, i.e., $2^{\text{level\_diff}}$ (see Lemma \ref{lem:up}). Practically, the number of resolutions $n$ is chosen as either 2 or 4, so the Lipschitz constant reaches a maximum of $2^3=8$, which is a factor causing the overall Lipschitz constant to explode. Lipschitz MDEQ inherits this operation as well.

\paragraph{Multiscale Fusion}
In the fusion layer of MDEQ, $\tilde{\bm{z}}_i$ is the sum of itself $\hat{\bm{z}}_i$ and all $\hat{\bm{z}}_j$ $(i \neq j)$ that have undergone downsample and upsample (see Eq. (M2)). However, since the Lipschitz constant for the upsampling operation in particular can become large depending on $\text{level}\_\text{diff}$, taking the sum may cause the Lipschitz constant to explode in the fusion layer.
Therefore, Lipschitz MDEQ first performs a convex combination between itself $\hat{\bm{z}}_i$ and the other upsampled or downsampled resolutions $\hat{\bm{z}}_j$ ($i\neq j$). Then, for upsampling and downsampling, it takes a weighted average, applying smaller weights as $\text{level}\_\text{diff}$ increases (see also Eq. (L2)), i.e.,
\begin{align*}
    \tilde{\bm{z}}_i = (1-\alpha_2)\hat{\bm{z}}_i + \alpha_2 \left( \sum_{j \neq i} w_{ij} \text{Fuse}_{i,j}(\hat{\bm{z}_i}) \right),
\end{align*}
where $\alpha_2 \in (0,1)$, $\text{Fuse}_{i,j}$ is $\text{D}_{i,j}$ when $j < i$ and $\text{U}_{i,j}$ when $j > i$, and the weights $w_{ij}$ are computed dynamically via a softmax function over penalties proportional to the $\text{level}\_\text{diff}$ of the upsampling paths:
\begin{align*}
    w_{ij} = \frac{\exp(-p_{ij})}{\sum_{k \neq i} \exp(-p_{ik})}, \text{where } p_{ij} = \begin{cases} j-i & \text{if } j > i \\ 0 & \text{if } j < i  \end{cases}.
\end{align*}
Crucially, the weights $w_{ij}$ are designed to penalize contributions from distant branches.
This mechanism ensures that upsampling operations with large $\text{level}\_\text{diff}$, and thus large Lipschitz constants, are assigned exponentially smaller weights. By adaptively down-weighting these potentially explosive terms, our fusion layer effectively suppresses the Lipschitz constant, contributing significantly to the overall stability.

\subsection{Lipschitz Constant of Lipschitz MDEQ}\label{sec:3.4}
The hidden state of MDEQ consists of a tuple of $n$ feature maps, $\bm{z} = [\bm{z}_1, \dots, \bm{z}_n]$, where each $\bm{z}_i$ can be viewed as a vector in a $d_i$-dimensional Euclidean space, $\mathbb{R}^{d_i}$. The full hidden state $\bm{z}$ therefore resides in the product space $\mathcal{Z} = \mathbb{R}^{d_1} \times \dots \times \mathbb{R}^{d_n}$.
The mapping $f_{\bm{\theta}}$ takes a tuple of hidden states $\bm{z} \in \mathcal{Z}$ and outputs $\bar{\bm{z}} = f_{\bm{\theta}}(\bm{z}) \in \mathcal{Z}$. We analyze the Lipschitz constant of $f_{\bm{\theta}}$ with respect to the norm on $\mathcal{Z}$ defined by
\begin{align*}
\| \bm{z} \|_{\mathcal{Z}} := \sqrt{\sum_{i=1}^n\| \bm{z}_i \|_2^2},
\end{align*}
where $\| \cdot \|_2$ denotes the standard Euclidean norm. This norm is equivalent to the $L_2$-norm of the vector formed by concatenating all hidden states $\bm{z}_i \in \mathbb{R}^{d_i}$, and naturally satisfies the axioms of a norm. 

\paragraph{Residual Block and Post-Fusion Layer} In the residual block, the branch processing function $\hat{f} \colon \mathcal{Z} \to \mathcal{Z}$ acts on each resolution branch independently. We can thus decompose it into $n$ component functions $\hat{f}_i: \mathbb{R}^{d_i} \to \mathbb{R}^{d_i}$, such that for an input $\bm{z} = [\bm{z}_1, \dots, \bm{z}_n]$, the output is simply $\hat{f}(\bm{z}) = [\hat{f}_1(\bm{z}_1), \ldots, \hat{f}_n(\bm{z}_n)]$. Note that although these functions differ in their input and output dimensions, the modules constituting them are identical. That is, $\hat{f}_i$ possesses the same Lipschitz constant for any $i \in [n]$, which we denote as $\hat{L}$. Then, we have
\begin{align*}
\| \hat{f}(\bm{z}) - \hat{f}(\bm{z}') \|_\mathcal{Z}^2
&= \sum_{i \in [n]} \| \hat{f}_i(\bm{z}_i) - \hat{f}_i(\bm{z}'_i) \|_2^2 \\
&\leq \sum_{i \in [n]} \hat{L}^2 \| \bm{z}_i - \bm{z}'_i \|_2^2 \\
&= \hat{L}^2 \| \bm{z} - \bm{z}' \|_\mathcal{Z}^2.
\end{align*}
Therefore, $\hat{f}$ also has the same Lipschitz constant, $\hat{L}$. 
From the definition of $\hat{f}_i$ (see Eq. (L1)) and Lemma \ref{lem:01}, it can be shown that 
\begin{align}\label{eq:hat}
\hat{L} &= (1-\alpha_1)L_{\text{MGN}}L_{\text{SReLU}} \nonumber \\
&\quad+\alpha_1 L_{\text{MGN}}^3L_{\text{Conv}^\star}^2 L_{\text{SReLU}}^2L_{\text{Drop}}.
\end{align}
Similarly, since the processing for each branch in the post-fusion layer $\bar{f}$ is completely independent, let the processing for each branch be denoted by $\bar{f}_i$ and the Lipschitz constant for these be $\bar{L}$. Then, the Lipschitz constant for $\bar{f}$ is also $\bar{L}$. From the definition of $\bar{f}_i$ (see Eq. (L3)), we have
\begin{align}\label{eq:bar}
\bar{L} = L_{\text{MGN}}L_{\text{Conv}^\star}L_{\text{SReLU}}.
\end{align}

\begin{figure*}[htbp]
\begin{minipage}[t]{1\linewidth}
\centering%
\includegraphics[width=1\linewidth]{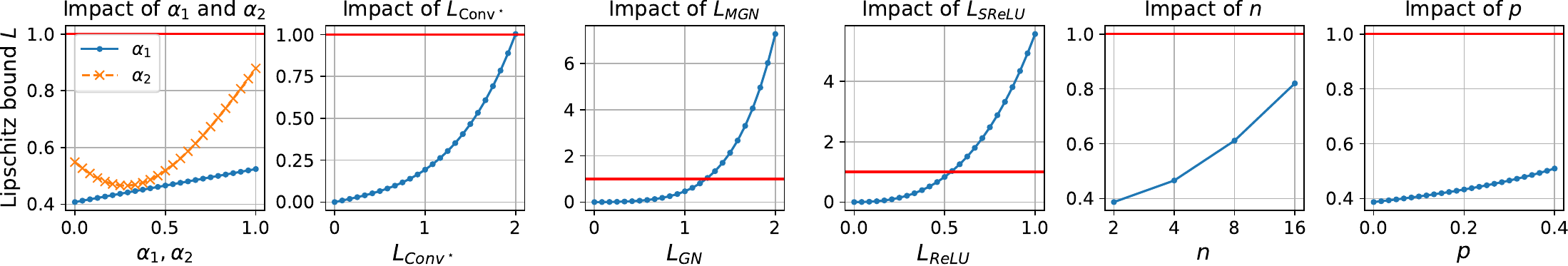}%
\end{minipage}%
\vspace*{-8pt}
\caption{Effect of hyperparameters constituting $f_{\bm{\theta}}$ on Lipschitz bound $L$. Each graph is plotted by varying only one parameter while holding all others constant. Fixed parameters are set as $\alpha_1=0.5, \alpha_2=0.3, L_{\text{Conv}^\star}=1.5, L_{\text{MGN}}=1.0, L_{\text{SReLU}}=0.4, n=4$, and $p=0.3$. Note that changing the values of the fixed parameters will also change the plot.}
\label{fig:sensi}
\vspace*{-10pt}
\end{figure*}

\paragraph{Fusion Layer}
In the fusion layer, the processing of each branch is not independent, as the hidden states influence each other.
Therefore, we define the function $\tilde{f}_i \colon \mathcal{Z} \to \mathbb{R}^{d_i}$ to take a tuple of $n$ hidden states $\hat{\bm{z}} \in \mathcal{Z}$ as input and output a single hidden state $\tilde{\bm{z}}_i \in \mathbb{R}^{d_i}$, i.e., $\tilde{f}(\hat{\bm{z}}) = [\tilde{f}_1(\hat{\bm{z}}), \ldots, \tilde{f}_n(\hat{\bm{z}})]$. Suppose $\tilde{f}_i$ is $\tilde{L}_i$-Lipschitz $(i \in [n])$, i.e., $\| \tilde{f}_i(\hat{\bm{z}}) - \tilde{f}_i(\hat{\bm{z}}') \|_2 \leq \tilde{L}_i \| \hat{\bm{z}} - \hat{\bm{z}}' \|_\mathcal{Z}$. Then, we have
\begin{align*}
\| \tilde{f}(\hat{\bm{z}}) - \tilde{f}(\hat{\bm{z}}') \|_\mathcal{Z}^2
&= \sum_{i \in [n]} \| \tilde{f}_i(\hat{\bm{z}}) - \tilde{f}_i(\hat{\bm{z}}') \|_2^2 \\
&\leq \sum_{i \in [n]} \tilde{L}_i^2 \| \hat{\bm{z}} - \hat{\bm{z}}' \|_\mathcal{Z}^2.
\end{align*}
Therefore, $\tilde{f}$ is $\sqrt{\sum_{i \in [n]} \tilde{L}_i^2}$-Lipschitz. From the definition of $\tilde{f}_i$ (see Eq. (L2)), we have
\begin{align}\label{eq:tilde}
\tilde{L}_i = \sqrt{(1-\alpha_2)^2 + \alpha_2^2 \sum_{j \neq i} (w_{ij}L_{\text{fuse}}^{i,j})^2}.
\end{align} 

Since $f_{\bm{\theta}}$ is the composite function of these three blocks, its Lipschitz constant is determined by their product and is equal to 
\begin{align}\label{eq:overall}
L=\hat{L}\sqrt{\sum_{i \in [n]} \tilde{L}_i^2}\bar{L},%
\end{align}
where $\hat{L}, \bar{L}, \tilde{L}_i$ are defined as in \eqref{eq:hat}, \eqref{eq:bar}, \eqref{eq:tilde}, respectively.

\paragraph{Sensitivity of Each Hyperparameter}
The Lipschitz constant for the forward pass is determined by the Lipschitz constants of each operation constituting $f_{\bm{\theta}}$, as given by equation \eqref{eq:overall}, and the Lipschitz constant of each operation is a hyperparameter. Specifically, it is determined by the two convex combination parameters $\alpha_1$ and $\alpha_2$, the target norm $L_{\text{Conv}^\star} = c$ in the convolution operation, the affine parameter $L_{\text{MGN}} = \bar{\gamma}$ in the normalization operation, the slope $L_{\text{SReLU}} = a$ of the activation function, the number of resolution branches $n$, and the dropout rate $p$. Figure \ref{fig:sensi} plots how much each parameter affects the Lipschitz constant $L$ for the entire forward pass. The magnitude of the impact depends on the order of the Lipschitz constant for each operation appearing within $L$, and it is evident that $L$ is particularly strongly influenced by $L_{\text{SReLU}}, L_{\text{MGN}}$, and $L_{\text{Conv}^\star}$. Furthermore, especially when these values are large, the small convex combination parameters $\alpha_1, \alpha_2$ play a crucial role in keeping $L$ small. For example, when $L_{\text{SReLU}}=0.35, L_{\text{Conv}^\star}=2.0, L_{\text{MGN}}=1.0, \alpha_1=0.5, \alpha_2=0.3, n=4$, and $p=0.3$, the Lipschitz constant for the forward pass is $L = 0.86$, satisfying the contraction property.

\subsection{Side Effects in Backward Pass}
As described already, we restructured the model architecture so that the fixed-point mapping $f_{\bm{\theta}}$ appearing in the forward pass (\ref{eq:forward}) becomes a contraction mapping; i.e., we ensured that the Lipschitz constant $L = \sup_{\bm{z} \in \mathcal{Z}} \| J_{f_{\bm{\theta}}}(\bm{z}) \|_2$ is less than 1, where $J_{f_{\bm{\theta}}}$ is the Jacobian matrix of $f_{\bm{\theta}}$.
This guarantees fixed-point convergence in the forward pass and is also expected to accelerate the forward pass.
Meanwhile, beneficial effects are also observed in the backward pass. The Lipschitz constant for the fixed-point mapping appearing in the backward pass (\ref{eq:back-fix}) is $\| J_{f_{\bm{\theta}}}(\bm{z^\star}) \|_2$ (see Lemma \ref{lem:back}), where $\bm{z}^\star \in \mathcal{Z}$ is the fixed point obtained in the forward pass. Since $\| J_{f_{\bm{\theta}}}(\bm{z^\star}) \|_2 \leq \sup_{\bm{z} \in \mathcal{Z}} \| J_{f_{\bm{\theta}}}(\bm{z}) \|_2$, the Lipschitz constant for the backward pass is always smaller than that in the forward pass. Therefore, forcing the forward pass Lipschitz constant to be less than 1 also means suppressing the backward pass Lipschitz constant to be less than 1. This guarantees fixed-point convergence for the backward pass as well, and acceleration of the backward pass is also anticipated. 

\section{Numerical Experiments}\label{sec:exp}
\begin{table*}[ht]
    \centering
    \caption{Comparison of different acceleration techniques for MDEQ on CIFAR-10. We report top-1 accuracy, total training speed, and average number of fixed-point iterations (NFEs) and runtime (Time) for the forward and backward passes. Best results in each column are highlighted in \textbf{bold}. JR: Jacobian regularization \cite{Bai2021Sta}; PG: phantom gradient \cite{Geng2021OnT}; and JFB: Jacobian-free backpropagation \cite{Fung2022JFB}. $^\dagger$Only JR has the maximum NFEs set to 7 and 8.}
    \label{tab:speedup}
    \resizebox{\textwidth}{!}{%
    \begin{tabular}{lcccccccccc}
        \toprule
        \multirow{2}{*}{\textbf{Model and method}} & \multirow{2}{*}{\textbf{Size}} & \multirow{2}{*}{\textbf{Speed-up}} & \multirow{2}{*}{\textbf{Acc.}} & \multicolumn{2}{c}{\textbf{Train forward}} & \multicolumn{2}{c}{\textbf{Train backward}} & \multicolumn{2}{c}{\textbf{Test forward}} \\
        \cmidrule(lr){5-6} \cmidrule(lr){7-8} \cmidrule(lr){9-10}
        & & & & NFEs & Time (ms) & NFEs & Time (ms) & NFEs & Time (ms) \\
        \midrule
        MDEQ & 10M &1.0$\times$ & 93.23 & 17.9 & 217 & 20 & 291 & 18.0 & 63.9 \\
        MDEQ + JR &10M &2.22$\times$ & 91.86 &7.0$^\dagger$ &74.8 &8.0$^\dagger$ &97.8 &7.0$^\dagger$ &39.2 \\
        MDEQ + PG &10M & 1.41$\times$ & 92.92 & 15.6 & 220 & -- & 87.6 & 18.0 & 63.9 \\
        MDEQ + JFB &10M & 2.84$\times$ & 88.13 & 12.3 & 142 & -- & \textbf{30.6} & 11.3 & 44.3 \\
        Lipschitz MDEQ ($L=14.43$) &10M& 1.28$\times$ & \textbf{93.73} & 16.0 & 187 & 19.9 & 193 & 15.5 & 54.1 \\
        Lipschitz MDEQ ($L=1.0$) &10M & 3.33$\times$ & 91.44 & 5.4 & 75.7 & 2.9 & 52.9 & 4.8 & 22.6 \\
        Lipschitz MDEQ ($L=0.03$) &10M & \textbf{4.75$\times$} & 90.47 & \textbf{2.0} & \textbf{44.0} & \textbf{2.0} & 42.1 & \textbf{2.0} & \textbf{13.3}  \\
        \midrule
        Lipschitz MDEQ ($L=14.43$) + JFB &10M & 3.66$\times$ & 88.66 & 7.6 & 90.18 & - & 24.63 & 7.11 & 26.51  \\
        Lipschitz MDEQ ($L=1.0$) + JFB &10M & 4.98$\times$ & 87.53 & 4.2 & 55.34& - & 24.64 & 4.0 & 17.57  \\
        Lipschitz MDEQ ($L=0.03$) + JFB &10M & 6.43$\times$ & 87.26 & 2.0 & 34.98 & - & 24.67 & 2.0 & 11.3  \\
        \bottomrule
    \end{tabular}}
    \vspace*{-10pt}
\end{table*}

For details on the experimental environment and hyperparameters, see Appendix \ref{sec:appexp}. In all experiments, we control the Lipschitz constant $L$ by fixing $L_{\text{MGN}}=1.0$, $L_{\text{Conv}^\star}=2.0$, $\alpha_1 = 0.5$, $\alpha_2=0.3$, $n=4$, and $p=0.3$ and varying the $L_{\text{SReLU}}=\{ 0.1, 0.2, 0.3, 0.4, 0.5, 0.6, 0.7, 0.8, 0.9, 1.0\}$. Specifically, $L=\{0.03, 1.0, 14.43\}$ are obtained when $L_{\text{SReLU}}=\{0.1, 0.4, 1.0\}$, respectively.

\paragraph{Comparison of Fixed-point Convergence}
In Lipschitz MDEQ, forcing the Lipschitz constant in the forward pass to be less than 1 guarantees the existence of a unique fixed point. Figure \ref{fig:convergence} plots fixed-point convergence for CIFAR-10 classification. Number of fixed-point iterations is fixed at 18 for the forward pass and 20 for the backward pass. The evaluation metric is relative residual $\| f_{\bm{\theta}}(\bm{z}; \bm{x}) - \bm{z} \|/\| f_{\bm{\theta}}(\bm{z};\bm{x})\|$. Throughout training, Lipschitz MDEQ achieves overwhelming fixed-point convergence, demonstrating better convergence than existing methods even in cases where convergence is not theoretically guaranteed (i.e., $L\geq1$). We obtained similar results for the forward pass during testing (see Appendix \ref{sec:exp_convergence}). 

\begin{figure}[htbp]
\begin{minipage}[t]{1\linewidth}
\centering%
\includegraphics[width=1\linewidth]{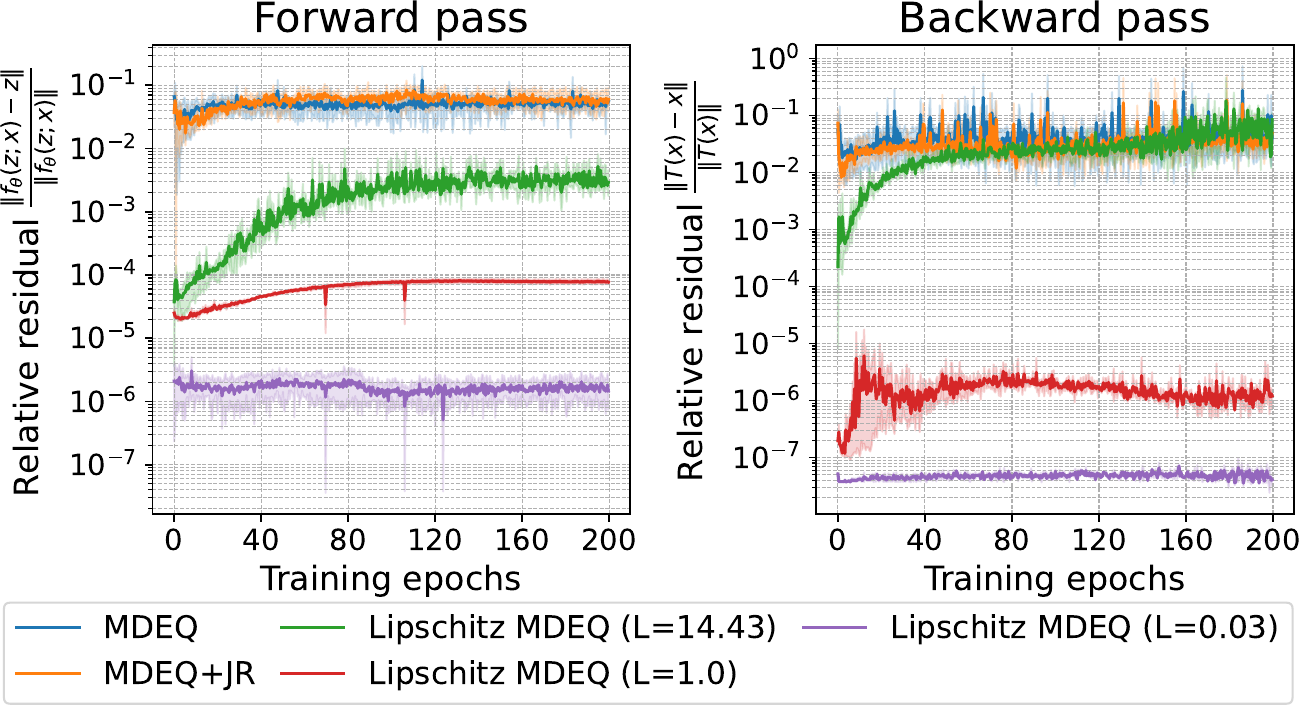}%
\end{minipage}%
\vspace*{-10pt}
\caption{Comparison of fixed-point convergence in forward and backward passes of training.}
\label{fig:convergence}
\vspace*{-10pt}
\end{figure}

\paragraph{Acceleration in Image Classification Tasks}
Table \ref{tab:speedup} summarizes the results of MDEQ and several acceleration methods for CIFAR-10 classification, along with our Lipschitz MDEQ.
All fixed-point iteration evaluation metrics used residual, with a threshold of 1e-3. The maximum iteration was set to 18 for the forward pass and 20 for the backward pass.
In both the forward pass during training and testing and the backward pass during training, Lipschitz MDEQ achieves significantly faster fixed-point convergence than existing methods, reducing runtime. However, it was observed that the stricter the contractive constraints, the more test accuracy deteriorated.
Furthermore, in theory, the smaller the Lipschitz constant, the better JFB works. Therefore, we also report the results of combining JFB with the Lipschitz MDEQ in Table \ref{tab:speedup}. This combination unlocks further speed improvements, but compared to Lipschitz MDEQ without JFB, the deterioration in accuracy is even more pronounced. 

\paragraph{Trade-off between Accuracy and Runtime}
Figure \ref{fig:trade_off} plots accuracy and speed as the Lipschitz constant is varied. The results shown demonstrate that controlling the Lipschitz constant allows the trade-off between accuracy and speed to be managed.
Furthermore, even when theoretical guarantees cannot be obtained ($L\geq1$), sufficient acceleration equivalent to or exceeding existing methods was achieved. In addition, by setting the Lipschitz constant to a large value without theoretical guarantees, our Lipschitz MDEQ achieves test accuracy equivalent to or exceeding that of MDEQ, along with a slight increase in speed.

\begin{figure}[htbp]
\begin{minipage}[t]{1\linewidth}
\centering%
\includegraphics[width=1\linewidth]{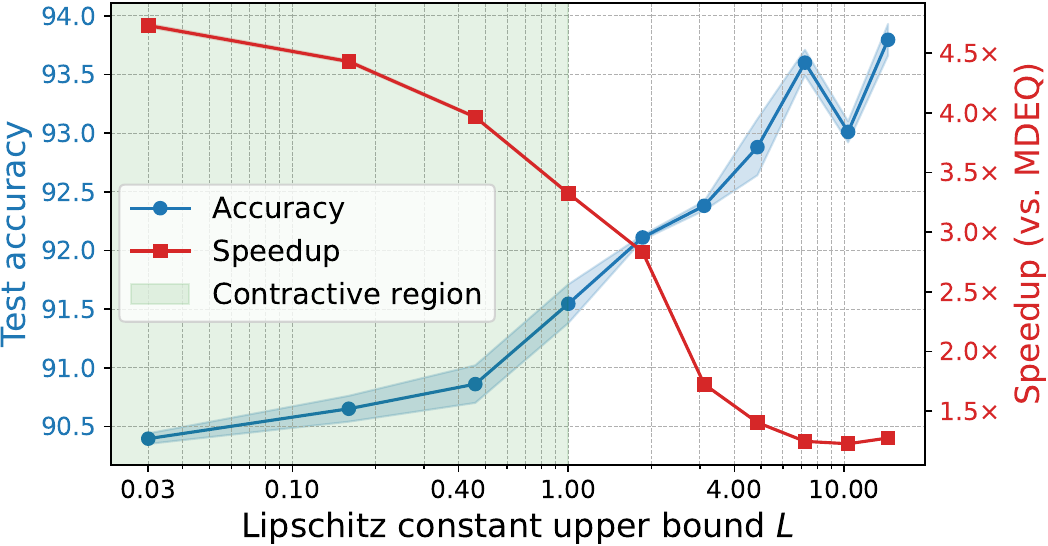}%
\end{minipage}%
\vspace*{-10pt}
\caption{Accuracy and speed trade-off in Lipschitz MDEQ. The region where convergence is guaranteed, $L < 1$, is shaded green.}
\label{fig:trade_off}
\vspace*{-10pt}
\end{figure}


\begin{table*}[htbp]
    \centering
    \caption{\UPDATE{Ablation study for Lipschitz MDEQ on CIFAR-10. We report top-1 accuracy, total training speed, and average number of fixed-point iterations (NFEs) and runtime (Time) for the forward and backward passes. Best results in each column are highlighted in \textbf{bold}.}}
    \label{tab:ablation}
    \resizebox{\textwidth}{!}{%
    \begin{tabular}{lcccccccccc}
        \toprule
        \multirow{2}{*}{\textbf{Model and method}} & \multirow{2}{*}{\textbf{Size}} & \multirow{2}{*}{\textbf{Speed-up}} & \multirow{2}{*}{\textbf{Acc.}} & \multicolumn{2}{c}{\textbf{Train forward}} & \multicolumn{2}{c}{\textbf{Train backward}} & \multicolumn{2}{c}{\textbf{Test forward}} \\
        \cmidrule(lr){5-6} \cmidrule(lr){7-8} \cmidrule(lr){9-10}
        & & & & NFEs & Time (ms) & NFEs & Time (ms) & NFEs & Time (ms) \\
        \midrule
        MDEQ & 10M &1.0$\times$ & 93.23 & 17.9 & 217 & 20 & 291 & 18.0 & 63.9 \\
        Lipschitz MDEQ ($L=1.0$) &10M & 3.33$\times$ & 91.44 & 5.4 & 75.7 & 2.9 & 52.9 & 4.8 & 22.6 \\
        (S1) Lipschitz MDEQ - MGN clip &10M& 2.81$\times$ & 91.99 & 6.4 & 90.5 & 4.7 & 65.3 & 5.7 & 25.2 \\
        (S2) Lipschitz MDEQ - MGN &10M & 0.95$\times$ & 92.78 & 18.0 & 225 & 20.0 & 296 & 18.0 & 72.9  \\
        (S3) Lipschitz MDEQ - Conv$^\star$ &10M & \textbf{3.42$\times$} & 91.94 & \textbf{3.8} & \textbf{61.4} & 5.6 & 73.0 & \textbf{3.7} & \textbf{19.0} \\
        (S4) Lipschitz MDEQ - softmax &10M & 3.41$\times$ & 91.29 & 5.4 & 76.5 & \textbf{2.8} & \textbf{48.2} & 4.9 & 20.6  \\
        (S5) Lipschitz MDEQ - $\alpha_1$ &10M & 2.72$\times$ & 92.35 & 6.6 & 90.9 & 5.5 & 71.9 & 6.0 & 25.9  \\
        (S6) Lipschitz MDEQ - $\alpha_2$ &10M & 1.64$\times$ & 92.73 & 9.7 & 120 & 18.9 & 179 & 9.8 & 33.2  \\
        (S7) Lipschitz MDEQ - $\alpha_1, \alpha_2$ &10M & 1.37$\times$ & \textbf{93.35} & 14.5 &169 & 19.6 & 184 & 14.5 & 45.4  \\
        \bottomrule
    \end{tabular}}
    \vspace*{-10pt}
\end{table*}

\UPDATE{
\subsection{Ablation Study and Discussion}\label{sec:ablation}
By systematically removing each modification applied to the individual operations that constitute Lipschitz MDEQ, we aim to verify the contribution and importance of each technique toward achieving both accuracy and speed. Table \ref{tab:ablation} summarizes the experimental results obtained under settings where only the modification of a single operation was undone. The default parameters inherit the values of Lipschitz MDEQ with $L=1.0$. 
The settings are as follows.
(S1) We removed the constraint $\bar{\gamma} \leq 1$ on the affine parameter of the normalization operation MGN.
(S2) We changed the normalization operation from MGN to GN.
(S3) We changed the convolution operation from Conv$^\star$ to general Conv.
(S4) We removed the weighted averaging mechanism via the softmax function in the fusion layer and replaced it with simple summation.
(S5) We removed the convex combination in the residual block and replaced it with pure residual connection.
(S6) We removed the convex combination in the fusion layer and replaced it with pure residual connection.
(S7) We removed the convex combinations in both the residual block and fusion layer and replaced both combinations with pure residual connections.

Table \ref{tab:ablation} shows that the most impactful modification in the transition from MDEQ to Lipschitz MDEQ was changing the normalization operation from GN to MGN (see (S2)). In (S2), despite techniques other than normalization still attempting to maintain a small overall Lipschitz constant, fixed-point convergence deteriorates significantly, resulting in the longest runtime. 
This result indicates that our modification is functioning as intended, and therefore we recommend using MGN rather than GN from the perspective of the Lipschitz constant of the normalization operation.

Furthermore, the results in (S5)--(S7) demonstrate that simple residual connections lead to worse fixed-point convergence compared to convex combinations.
Similarly, while this result validates our theory, note that using convex combinations results in a deterioration of accuracy.
There may be room for improvement in achieving high test accuracy and speed-up through superior fixed-point convergence.

In addition, the results of (S3) and (S4) could not be predicted by our theoretical analysis. First, the results of (S4) show that while taking a weighted average using the softmax function in the fusion layer undoubtedly contributes to theoretically suppressing the overall Lipschitz constant, the experimental benefits were limited. Next, (S3) demonstrates that imposing no constraints on the Lipschitz constant of the convolution operation does not worsen fixed-point convergence; rather, it partially improves it.
Moreover, since parameter norm constraints naturally reduce a model's expressive power, (S3) achieves higher test accuracy than the full-specification Lipschitz MDEQ (second row in the table). Therefore, we must conclude that the hyperparameter setting for Lipschitz MDEQ that removes only the constraint on the convolution operation (S3) is practical in terms of both test accuracy and runtime. Of course, it is important to note that the acceleration of fixed-point convergence in (S3) is not fully guaranteed theoretically.
Thus, Lipschitz MDEQ still exhibits several gaps between theory and experimental observations, pointing to one direction for future work.

Finally, to accelerate DEQ based on our approach, theoretically any modification is acceptable as long as the Lipschitz constant can be kept small; our model modification is merely one example.
While this study has demonstrated the effectiveness of this approach, important future work will also involve exploring model modifications that suppress the Lipschitz constant without compromising test accuracy, as well as modifications effective for other explicit models.
}

\section{Conclusion}
This paper proposes Lipschitz MDEQ, a modified architecture of MDEQ. The proposed model allows users to determine the Lipschitz constant of the fixed-point mapping in the forward pass, which can thereby be made into a contracting mapping, guaranteeing fixed-point convergence. This dramatically improves fixed-point convergence in the forward pass.
Furthermore, since the Lipschitz constant of the fixed-point mapping in the backward pass is smaller than that in the forward pass, the acceleration of fixed-point convergence extends to the backward pass as well. Consequently, the proposed architecture can dramatically reduce the time required for training and inference compared to existing models and acceleration techniques, albeit at the cost of slightly lower accuracy.
We are confident that our results represent a significant step toward developing high-performance, fast models with superior memory efficiency, addressing the critical drawback of the low learning and inference speeds in DEQ.

\UPDATE{
\section*{Acknowledgements}
We are sincerely grateful to Program Chairs, Area Chairs, and three anonymous reviewers for helping us improve the original manuscript. 
We would like to thank Junya Morioka for his technical support in the numerical experiments.
This work was supported by the Japan Society for the Promotion of Science (JSPS) KAKENHI Grant Number 24K14846 awarded to Hideaki Iiduka.
}

\bibliography{main}

\clearpage
\appendix
\thispagestyle{empty}

\onecolumn
\aistatstitle{Lipschitz Multiscale Deep Equilibrium models:\\ A Theoretically Guaranteed and Accelerated Approach \\
Supplementary Materials}


\section{Other Related Works}\label{sec:other}
\paragraph{\UPDATE{Other Theory}}
Other research on theoretical aspects of DEQs includes work related to neural tangent kernel theory \cite{Feng2023Ont, Agarwala2022Dee, Gao2022Agl, Ling2023Glo, Truong2025Glo, Ling2024Dee}. Additionally, Kawaguchi proved that under certain assumptions, gradient descent converges to a global optimum for minimizing empirical loss in linear DEQs \cite{Kawaguchi2021Ont}.
\UPDATE{Also, the gradient descent dynamics of DEQ has been studied in the simple settings of linear and single-index models \cite{Dandapanthula2025Gra}.
Gao et al. explored the infinite-width behavior of deep equilibrium neural networks, and found a limiting Gaussian process behavior that occurs even when depth and width limits are interchanged \cite{Gao2023Wide}.}

\paragraph{Application}
The defining feature of a DEQ is its ability to achieve both high performance and low memory consumption, making it applicable to a remarkably diverse range of tasks. Particularly notable is its application to inverse problems such as wireless channel estimation \cite{Yuan2024APe, Tian2024GSU, Tian2025Uns}, MRI and general imaging reconstruction \cite{Gilton2021Dee, Liu2022Onl, Pramanik2022Imp, Zou2023Dee, Park2025Eff, Kuo2025Mat}, CT reconstruction \cite{Bubba2025Tom}, \UPDATE{self-supervised imaging reconstruction \cite{Mehta2025Equ}}, MPI reconstruction \cite{Alper2024DEQ}, video snapshot compressive imaging \cite{Zhao2023Dee}, Poisson inverse problems \cite{Daniele2025Dee}, Plug-and-Play image reconstruction \cite{Chandler2023Ove}, medical image registration \cite{HU2024APl, Zhang2025Bri}, imaging photoplethysmography (iPPG) \cite{Shenoy2025Rec}, inverse problems under model mismatch \cite{Shoushtari2022Dee, Hu2023Rob, Guan2024Sol}, and music source separation \cite{Koyama2022Mus}. 
Other applications include diffusion models \cite{Pokle2022Dee, Geng2023One, Bai2024Fix}, image restoration \cite{Chen2023Equ, Jiezhang2024Dee}, object detection \cite{Wang2020Imp, Wang2023Dee}, landmark detection \cite{Micaelli2023Rec}, optical flow estimation \cite{Bai2022Dee}, point cloud classification and point cloud completion \cite{Geuter2025DDE}, learning solution operators for steady-state PDEs \cite{Marwah2023Dee}, parameterized quantum circuits \cite{Schleich2024Qua}, quantum Hamiltonian prediction \cite{Wang2024Inf}, \UPDATE{molecular dynamics simulations \cite{Burger2025DEQ},} initial value problems of ordinary differential equations \cite{Pacheco2024Sol}, image compressive sensing \cite{Yu2024MsD}, energy minimization in Hopfield networks \cite{Goemaere2024Acc}, multimodal fusion \cite{Jinhong2023Dee}, federated learning \cite{Gkillas2023Dee}, and video semantic segmentation and video object detection \cite{Ertenl2022Str}. Geng et al.\ have provided a convenient PyTorch implementation of DEQ \cite{Geng2023Tor}.

\paragraph{Adversarial Robustness}
Many previous studies focused on the adversarial robustness of DEQ. Chen et al.\ examined the robustness of monDEQ using its semi-algebraic representation \cite{Chen2021Sem}. Pabbaraju et al.\ estimated the Lipschitz constant of monDEQ and demonstrated its robustness against adversarial attacks \cite{Pabbaraju2021Est}. Jafarpour et al.\ studied the stability of fixed-point convergence in a DEQ theoretically using non-Euclidean contraction theory \cite{Jafarpour2021Rob}. Wei and Kolter proposed a new DEQ called IBP-monDEQ that preserves the existence and uniqueness of fixed points for monDEQ while also providing theoretical guarantees for robustness from the perspective of interval bound propagation \cite{Wei2022Cer}. Li et al.\ proposed a certifiable DEQ called CerDEQ \cite{Li2022Cer}. Yang et al.\ examined correct attack methods to avoid misjudging the adversarial robustness of DEQs \cite{Yang2022ACl}. Havens et al.\ verified the $\ell2$-certified robustness of DEQs from the perspective of Lipschitz boundaries \cite{Havens2023Exp}. Jafarpour et al.\ evaluated the robustness of implicit models using mixed monotone systems theory and contraction theory \cite{Jafarpour2022Jaf}. Yang et al.\ improved the robustness of DEQs from the perspective of neural dynamics \cite{Yang2023Imp}.
Chu et al.\ proposed LyaDEQ, a robust DEQ, using Lyapunov functions \cite{Chu2024Lya}. Gao et al.\ generalized provable robustness to large-scale datasets \cite{Gao2024Cer}.

\paragraph{Implicit Layers Family}
DEQs are positioned within the broader family of implicit layers.
Implicit layers is a class of models whose outputs are defined implicitly as the solution to an equation or optimization problem, rather than being computed through a fixed sequence of operations. The history of implicit layers in deep learning dates back to the late 1980s, originating from a paper by Pineda and Almeida \cite{Pineda1987Gen, Almeida1987Ale}, where it is referred to as recurrent back-propagation  \cite{Liao2018Rev}. 
The family includes Neural ODEs (solutions to differential equations) \cite{Chen2018Neu, Grathwohl2019FFJ, Dupont2019Aug}, differentiable physics engines \cite{Avila2018End, Qiao2020Sca}, and differentiable optimization layers (solutions to optimization problems) such as OptNet \cite{Amos2017Opt}, differentiable submodular models \cite{Djolonga2017Dif, Tschiatschek2018Dif}, SparseMAP \cite{Niculae2018Spa}, SATNet \cite{Wang2019SAT}, and differentiable convex optimization \cite{Agrawal2019Dif}.
More generally, the deep declarative networks framework \cite{Gould2022Dee} provides a unifying theory that encompasses such differentiable optimization layers, defining network outputs as solutions to optimization problems. Furthermore, around the same time as DEQs, Ghaoui et al.\ proposed an implicit deep learning framework defining outputs via fixed-point equations and demonstrated its advantages in robustness and interpretability \cite{Ghaoui2021Imp}.
These models often exhibit compatible functionalities. For example, Ding et al.\ demonstrated that DEQs and NeuralODE are two sides of the same coin \cite{Ding2023Two}, while Pal et al.\ proposed a new architecture that redefines DEQ as an infinite-time neural ODE to reduce training costs \cite{Pal2023Con}. In addition, there are several applications, such as graph neural networks \cite{Gu2020Imp, Liu2021EIG, Park2022Con, Chen2022Opt} and generative models \cite{Lu2021Imp}.

\paragraph{Lipschitz Networks}
Our approach is of the kind that controls the model's Lipschitz constant, and there have been several prior studies for explicit models \cite{Zhao2021Exp, Bethune2022Pay}. Existing approaches for enforcing Lipschitz constraints fall into two categories: regularization constraints and architecture constraints. Regularization approaches \cite{Drucker1992Imp, Gulrajani2017Imp} exhibit good practical performance but do not reliably enforce Lipschitz constraints. Architecture-based approaches impose restrictions on the operator norm of weights \cite{Moustapha2017Par, Miyato2018Spe}. These reliably satisfy the Lipschitz constraint but often come at the cost of reduced expressiveness. Anil et al.\ proposed a new activation function intended to make neural networks into Lipschitz functions without sacrificing expressiveness \cite{Anil2019Sor}. In the context of DEQs, Jacobian regularization \cite{Bai2021Sta} corresponds to a regularization approach, while our study's approach is architecture-based. 
\section{Derivation of Lipschitz Constant for Each Operation}\label{sec:proof}
\subsection{Normalization}\label{sec:norm}
\begin{lem}\label{lem:gn}
Let $\text{GN} \colon \mathbb{R}^d \to \mathbb{R}^d$ be the group normalization operation, i.e., 
\begin{align*}
GN(\bm{z})_i = \gamma \frac{z_i - \mu_{\bm{z}}}{\sqrt{\sigma_{\bm{z}}^2 + \epsilon}} + \beta,
\end{align*}
where $\mu_{\bm{z}} := \frac{1}{d}\sum_{i \in [d]} z_i$ and $\sigma_{\bm{z}}^2 := \frac{1}{d}\sum_{i \in [d]} (z_i - \mu_{\bm{z}})^2$ are the mean and variance, and $\gamma, \beta \in \mathbb{R}$ are learnable affine parameters. Then, $\text{GN}$ is a $\frac{|\gamma|}{\sqrt{\epsilon}}$-Lipschitz function.
\end{lem}
\begin{proof}
If the mapping $f \colon \mathbb{R}^d \to \mathbb{R}^d$ is differentiable, then the Lipschitz constant $L_f$ of $f$ is characterized by the operator norm of the Jacobian matrix $J$, and $L_f=\sup_{\bm{x} \in \mathbb{R}^d} \| J(\bm{x}) \|_2$ holds.
Therefore, we first show that the operation of GN is differentiable.

All elementary operations used to define $\text{GN}$ are differentiable on their domains and remain well-defined because $\epsilon>0$. Concretely, the following are true:
\begin{itemize}
  \item The maps $\bm{z} \mapsto \mu_{\bm{z}}=\frac{1}{d}\bm{1}^\top \bm{z}$ and $\bm{z} \mapsto \bm{z}-\mu_{\bm{z}}\bm{1}$ are linear, hence differentiable to every order.
  \item The map $\bm{z}-\mu_{\bm{z}}\bm{1} \mapsto (\bm{z}-\mu_{\bm{z}}\bm{1})^\top (\bm{z}-\mu_{\bm{z}}\bm{1})$ is a polynomial, hence differentiable to every order.
  \item Since $\sigma_{\bm{z}}^2=\frac{1}{d}(\bm{z}-\mu_{\bm{z}}\bm{1})^\top (\bm{z}-\mu_{\bm{z}}\bm{1}) \geq 0$ and $\epsilon>0$, the scalar function $t\mapsto\sqrt{t+\epsilon}$ is smooth on the interval $[0,\infty)$; in particular, it admits derivatives of all orders when evaluated at $t=\sigma_{\bm{z}}^2$.
\end{itemize}
Note that $\bm{1}$ denotes a vector whose elements are all 1. Therefore, $\mathrm{GN}$ is differentiable at every $\bm{z}\in\mathbb{R}^d$, and in fact derivatives of arbitrary order exist (because each constituent map admits derivatives of arbitrary order and derivatives are closed under composition). Moreover, the denominator
\begin{align*}
s(\bm{z}):=\sqrt{\sigma_{\bm{z}}^2+\epsilon}
\end{align*}
satisfies $s(\bm{z})\geq\sqrt{\epsilon}>0$ for all $\bm{z}$, so no division by zero occurs; consequently, every partial derivative of the Jacobian entries is continuous and the Jacobian matrix $J(\bm{z})$ depends continuously on $\bm{z}$.

Next, we derive the Jacobian of GN. For simplicity, we introduce $\hat{z}_i := \frac{z_i - \mu_{\bm{z}}}{\sqrt{\sigma_{\bm{z}}^2+\epsilon}}$ and obtain $\text{GN}(\bm{z})_i := \gamma \hat{z}_i + \beta$. 
As preparation, we will calculate several derivatives. From the definition of $\mu_{\bm{z}}$, we have
\begin{align*}
\frac{\partial \mu_{\bm{z}}}{\partial z_k} = \frac{\partial}{\partial z_k} \left( \frac{1}{d} \sum_{i \in [d]} z_i \right) = \frac{1}{d}, \ \text{and }\ 
\frac{\partial \mu_{\bm{z}}^2}{\partial z_k} = \frac{\partial \mu_{\bm{z}}^2}{\partial \mu_{\bm{z}}} \cdot \frac{\partial \mu_{\bm{z}}}{\partial z_k} = \frac{2\mu_{\bm{z}}}{d}.
\end{align*}
In addition, from the definition of $\sigma_{\bm{z}}^2$, we have
\begin{align*}
\frac{\partial \sigma_{\bm{z}}^2}{\partial z_k} = \frac{\partial}{\partial z_k}\left( \frac{1}{d} \sum_{i \in [d]}z_i^2 - \mu_{\bm{z}}^2 \right)
= \frac{\partial}{\partial z_k}\left( \frac{1}{d} \sum_{i \in [d]} z_i^2 \right) - \frac{\partial}{\partial z_k}\mu_{\bm{z}}^2
= \frac{2}{d}(z_k - \mu_{\bm{z}}).
\end{align*}
Moreover, from the definition of $s(\bm{z})$, 
\begin{align*}
\frac{\partial s(\bm{z})}{\partial z_k} 
= \frac{\partial s(\bm{z})}{\partial \sigma_{\bm{z}}^2}\cdot \frac{\partial \sigma_{\bm{z}}^2}{\partial z_k}
= \frac{(\sigma_{\bm{z}}^2+\epsilon)^{-\frac{1}{2}}}{2} \cdot \frac{2}{d}(z_k - \mu_{\bm{z}})
=\frac{\hat{z}_k}{d}.
\end{align*}
Finally, from the quotient rule, 
\begin{align*}
\frac{\partial \hat{z}_i}{\partial z_k}
&= \frac{\partial}{\partial z_k} \frac{z_i - \mu_{\bm{z}}}{s(\bm{z})}
= \frac{\frac{\partial}{\partial z_k}(z_i - \mu_{\bm{z}})s(\bm{z}) - (z_i - \mu_{\bm{z}})\frac{\partial}{\partial z_k} s(\bm{z})}{s(\bm{z})^2}
= \frac{\left(\frac{\partial}{\partial z_k}z_i - \frac{\partial}{\partial z_k} \mu_{\bm{z}}\right)s(\bm{z}) - (z_i - \mu_{\bm{z}})\frac{\hat{z}_k}{d}}{s(\bm{z})^2} \\
&=\frac{\left(\delta_{ik} - \frac{1}{d} \right)s(\bm{z}) - (z_i - \mu_{\bm{z}})\frac{\hat{z}_k}{d}}{s(\bm{z})^2}
= \frac{1}{s(\bm{z})}\left(\delta_{ik} - \frac{1}{d} \right) - \frac{1}{s(\bm{z})}\frac{z_i - \mu_{\bm{z}}}{s(\bm{z})}\frac{\hat{z}_k}{d}
= \frac{1}{s(\bm{z})}\left\{ \left(\delta_{ik} - \frac{1}{d}\right) - \frac{1}{d} \hat{z}_i\hat{z}_k\right\},
\end{align*}
where $\delta_{ik}$ is the Kronecker delta. Then, from the chain rule, we have
\begin{align*}
J_{ik} = \frac{\partial \text{GN}(\bm{z})_i}{\partial z_k} 
= \frac{\partial \text{GN}(\bm{z})_i}{\partial \hat{z}_i} \cdot \frac{\partial \hat{z}_i}{\partial z_k}
=\gamma \frac{\partial \hat{z}_i}{\partial z_k}
=\frac{\gamma}{s(\bm{z})}\left\{ \left(\delta_{ik} - \frac{1}{d}\right) - \frac{1}{d} \hat{z}_i\hat{z}_k\right\}.
\end{align*}
Therefore, 
\begin{align*}
J(\bm{z}) = \frac{1}{s(\bm{z})}\operatorname{diag}(\gamma)\left\{ \left( I - \frac{1}{d}\bm{1}\bm{1}^\top \right) - \frac{1}{d} \hat{\bm{z}}\hat{\bm{z}}^\top \right\}.
\end{align*}
From submultiplicativity, we have
\begin{align*}
\| J(\bm{z}) \|_2 \leq \left\| \frac{1}{s(\bm{z})} \operatorname{diag}(\gamma) \right\|_2 \left\| \left( I - \frac{1}{d}\bm{1}\bm{1}^\top \right) - \frac{1}{d} \hat{\bm{z}}\hat{\bm{z}}^\top \right\|_2 
\leq \frac{|\gamma|}{s(\bm{z})}\Big\| \underbrace{\left( I - \frac{1}{d}\bm{1}\bm{1}^\top \right) - \frac{1}{d} \hat{\bm{z}}\hat{\bm{z}}^\top }_{=: A(\bm{z})} \Big\|_2.
\end{align*}
Since $\| A(\bm{z}) \|_2$ is the largest eigenvalue of $A(\bm{z})$, let us consider the eigenvalues of $A(\bm{z})$.
\begin{enumerate}[label=(\roman*)]
  \item Multiplying $A(\bm{z})$ by $\hat{\bm{z}}$ gives
  \begin{align*}
  A(\bm{z})\hat{\bm{z}} 
  = \left( I - \frac{1}{d} \bm{1}\bm{1}^\top \right)\hat{\bm{z}} - \frac{1}{d}\hat{\bm{z}}\hat{\bm{z}}^\top\hat{\bm{z}}
  = \left( \hat{\bm{z}} - \frac{1}{d} \bm{1}\left(\bm{1}^\top \hat{\bm{z}}\right) \right) - \frac{1}{d} \hat{\bm{z}} \| \hat{\bm{z}} \|_2^2
  = \left( 1- \frac{\| \hat{\bm{z}}\|_2^2}{d}\right) \hat{\bm{z}},
  \end{align*}
  where we use $\bm{1}^\top \hat{\bm{z}} = 0$. Here, from the definition of $\hat{\bm{z}}$, we have
  \begin{align*}
  \| \hat{\bm{z}} \|_2^2 
  = \left\| \frac{\bm{z} - \mu_{\bm{z}} \bm{1}}{\sqrt{\sigma^2 + \epsilon}} \right\|_2^2
  = \frac{1}{\sigma_{\bm{z}}^2+\epsilon} \| \bm{z}- \mu_{\bm{z}}\bm{1}\|_2^2
  = \frac{1}{\sigma_{\bm{z}}^2+\epsilon} \sum_{i \in [d]} (z_i - \mu_{\bm{z}})^2 
  = \frac{d\sigma_{\bm{z}}^2}{\sigma_{\bm{z}}^2 + \epsilon}.
  \end{align*}
  Hence, 
  \begin{align*}
  A(\bm{z})\hat{\bm{z}} 
  = \left( 1- \frac{\sigma_{\bm{z}}^2}{\sigma_{\bm{z}}^2+\epsilon}\right)\hat{\bm{z}}
  = \frac{\epsilon}{\sigma_{\bm{z}}^2 + \epsilon} \hat{\bm{z}}.
  \end{align*}
  Therefore, $\hat{\bm{z}}$ is an eigenvector, and its eigenvalue is $\frac{\epsilon}{\sigma_{\bm{z}}^2 + \epsilon} \in (0,1]$.
  \item Multiplying $A(\bm{z})$ by $\bm{1}$ gives
  \begin{align*}
  A(\bm{z})\bm{1}
  = \left(I - \frac{1}{d}\bm{1}\bm{1}^\top \right)\bm{1} - \frac{1}{d}\hat{\bm{z}}\hat{\bm{z}}^\top\bm{1}
  = \bm{1} - \frac{1}{d}\bm{1}\left(\bm{1}^\top\bm{1}\right) - \frac{1}{d}\hat{\bm{z}}\left( \hat{\bm{z}}^\top \bm{1}\right)
  = \bm{1} - \frac{1}{d}\bm{1}\cdot d - \frac{1}{d}\hat{\bm{z}}\cdot 0 = \bm{0},
  \end{align*}
  where we use $\bm{1}^\top \hat{\bm{z}} = 0$. Therefore, $\bm{1}$ is an eigenvector, and its eigenvalue is $0$.
  \item Multiplying $A(\bm{z})$ by a vector $\bm{v}$ that is perpendicular to both $\hat{\bm{z}}$ and $\bm{1}$,
  \begin{align*}
  A(\bm{z})\bm{v}
  = \left( I - \frac{1}{d} \bm{1}\bm{1}^\top \right)\bm{v} - \frac{1}{d}\hat{\bm{z}}\hat{\bm{z}}^\top \bm{v} 
  = \bm{v} - \frac{1}{d}\bm{1}\left( \bm{1}^\top \bm{v}\right) - \frac{1}{d}\hat{\bm{z}}\left( \hat{\bm{z}}^\top \bm{v} \right)
  = \bm{v} - \frac{1}{d} \bm{1} \cdot 0 - \frac{1}{d}\hat{\bm{z}}\cdot 0 = \bm{v},
  \end{align*}
  where we use $\bm{1}^\top \bm{v} = 0$ and $\hat{\bm{z}}^\top \bm{v} = 0$. Therefore, $\bm{v}$ is an eigenvector, and its eigenvalue is $1$.
\end{enumerate}
Hence, the largest eigenvalue of $A(\bm{z})$ is $1$, i.e., $\| A(\bm{z}) \|_2 = 1$. Consequently, we have
\begin{align*}
\| J(\bm{z}) \|_2 \leq \frac{|\gamma|}{s(\bm{z})} 
= \frac{|\gamma|}{\sqrt{\sigma_{\bm{z}}^2 + \epsilon}} 
\leq \frac{|\gamma|}{\sqrt{\epsilon}}.
\end{align*}
This completes the proof.
\end{proof}

\begin{lem}\label{lem:mgn}
Let $\text{MGN} \colon \mathbb{R}^d \to \mathbb{R}^d$ be the mean-only group normalization operation, i.e., 
\begin{align*}
MGN(\bm{z})_i = \gamma (z_i - \mu_{\bm{z}}) + \beta,
\end{align*}
where $\mu_{\bm{z}} := \frac{1}{d}\sum_{i \in [d]} z_i$ is the mean, and $\gamma, \beta \in \mathbb{R}$ are learnable affine parameters. Then, $\text{MGN}$ is a $|\gamma|$-Lipschitz function.
\end{lem}
\begin{proof}
By the same argument as Lemma \ref{lem:gn}, since MGN is differentiable, it suffices to consider its Jacobian matrix:
\begin{align*}
J_{ik} 
= \frac{\partial \text{MGN}(\bm{z})_i}{\partial z_k} 
= \frac{\partial}{\partial z_k} \left\{ \gamma(z_i - \mu_{\bm{z}}) + \beta \right\}
= \gamma \delta_{ik} - \frac{1}{d},
\end{align*}
where $\delta_{ik}$ is the Kronecker delta. Therefore, 
\begin{align*}
J(\bm{z}) = \operatorname{diag}(\gamma) \left(I - \frac{1}{d}\bm{1}\bm{1}^\top \right).
\end{align*}
From submultiplicativity, we have
\begin{align*}
\| J(\bm{z}) \|_2 \leq \left\|\operatorname{diag}(\gamma) \right\|_2 \left\|I-\frac{1}{d}\bm{1}\bm{1}^\top\right\|_2
=|\gamma| \Big\| \underbrace{I-\frac{1}{d}\bm{1}\bm{1}^\top}_{=:P} \Big\|_2.
\end{align*}
Let us consider the eigenvalues of $P$. Here, from
\begin{align*}
P^\top 
= \left( I - \frac{1}{d}\bm{1}\bm{1}^\top \right)^\top 
= I - \frac{1}{d}\bm{1}\bm{1}^\top 
= P,
\end{align*}
and
\begin{align*}
P^2 
= \left( I - \frac{1}{d}\bm{1}\bm{1}^\top \right) \left( I - \frac{1}{d}\bm{1}\bm{1}^\top \right)
= I -\frac{1}{d}\bm{1}\bm{1}^\top - \frac{1}{d}\bm{1}\bm{1}^\top + \frac{1}{d^2}\bm{1}\bm{1}^\top \bm{1}\bm{1}^\top
= I -\frac{1}{d} \bm{1}\bm{1}^\top 
=P,
\end{align*}
$P$ is a projection matrix, and its eigenvalues are $0$ and $1$. Therefore, we have $\| P \|_2 = 1$ and
\begin{align*}
\| J(\bm{z}) \|_2 \leq |\gamma|.
\end{align*}
This completes the proof.
\end{proof}

\subsection{Activation Function}\label{sec:relu}
\begin{lem}\label{lem:relu}
Let $\text{ReLU} \colon \mathbb{R}^d \to \mathbb{R}^d$ be the ReLU function, i.e., $\text{ReLU}(\bm{z})_i := \max\{ 0, z_i \}$. Then, $\text{ReLU}$ is a $1$-Lipschitz function.
\end{lem}
\begin{proof}
As preparation, consider the scalar version of the ReLU function $g \colon \mathbb{R} \to \mathbb{R}$, i.e., $g(z) := \max\{ 0, z\}$.
\begin{itemize}
\item If $z_1 \geq 0$ and $z_2 \geq 0$, then we have
\begin{align*}
| g(z_1) - g(z_2) | = | z_1 - z_2|. 
\end{align*}
\item If $z_1 \leq 0$ and $z_2 \leq 0$, then we have
\begin{align*}
| g(z_1) - g(z_2) | = 0 \leq | z_1 - z_2|. 
\end{align*}
\item If $z_1 \geq 0$ and $z_2 \leq 0$, then we have
\begin{align*}
| g(z_1) - g(z_2) | = | z_1 - 0| = z_1 \leq z_1 - z_2 = |z_1 - z_2|. 
\end{align*}
\item If $z_1 \leq 0$ and $z_2 \geq 0$, then we have
\begin{align*}
| g(z_1) - g(z_2) | = | 0 - z_2 | = z_2 \leq z_2 - z_1 \leq |z_1 - z_2|. 
\end{align*}
\end{itemize}
Hence, for all $z_1, z_2 \in \mathbb{R}$,
\begin{align*}
|g(z_1) - g(z_2)| \leq |z_1 - z_2|.
\end{align*}
From the scalar result, we have for each coordinate
\begin{align*}
|\text{ReLU}(\bm{x})_i - \text{ReLU}(\bm{y})_i | \leq |x_i - y_i|.
\end{align*}
Therefore, we have
\begin{align*}
\| \text{ReLU}(\bm{x}) - \text{ReLU}(\bm{y}) \|_2 
= \left( \sum_{i \in [d]} | \text{ReLU}(\bm{x})_i - \text{ReLU}(\bm{y})_i |^2 \right)^{\frac{1}{2}}
\leq \left(\sum_{i \in [d]} |x_i - y_i|^2 \right)^{\frac{1}{2}}
= \| \bm{x} - \bm{y} \|_2.
\end{align*}
This completes the proof.
\end{proof}

\begin{lem}\label{lem:srelu}
Let $\text{SReLU} \colon \mathbb{R}^d \to \mathbb{R}^d$ be the SReLU function, i.e., $\text{SReLU}(\bm{z})_i := \max\{ 0, az_i \}$ $(a \in (0,1])$. Then, $\text{SReLU}$ is an $a$-Lipschitz function.
\end{lem}
\begin{proof}
By a similar argument to Lemma \ref{lem:relu}, we have for each coordinate
\begin{align*}
|\text{SReLU}(\bm{x})_i - \text{SReLU}(\bm{y})_i | \leq a |x_i - y_i|.
\end{align*}
Therefore, we have
\begin{align*}
\| \text{SReLU}(\bm{x}) - \text{SReLU}(\bm{y}) \|_2
= \left( \sum_{i \in [d]} | \text{SReLU}(\bm{x})_i - \text{SReLU}(\bm{y})_i |^2 \right)^{\frac{1}{2}}
\leq \left(\sum_{i \in [d]} a^2|x_i - y_i|^2 \right)^{\frac{1}{2}}
= a\| \bm{x} - \bm{y} \|_2.
\end{align*} 
This completes the proof.
\end{proof}

\subsection{Dropout}\label{sec:drop}
\begin{lem}\label{lem:drop}
Let $\text{Drop} \colon \mathbb{R}^d \to \mathbb{R}^d$ be the variational dropout function, i.e., $\text{Drop}(\bm{z}) := \frac{1}{1-p}\bm{m} \odot \bm{z}$, where $p \in (0,1)$ and $\bm{m} \sim \text{Bernouli}(p)$. Then $\text{Drop}$ is a $\frac{1}{1-p}$-Lipschitz function.
\end{lem}
\begin{proof}
Fix a realization of the mask $\bm{m} = (m_1, \ldots, m_d)^\top$ with $m_i \in \{0,1 \}$. For each coordinate $i \in [d]$, we have
\begin{align*}
|\text{Drop}(\bm{x})_i - \text{Drop}(\bm{y})_i | 
= \frac{1}{1-p}|m_i (x_i - y_i)| 
\leq \frac{1}{1-p}|x_i - y_i|,
\end{align*}
where we use $|m_i| \leq 1$. Therefore, we have
\begin{align*}
\| \text{Drop}(\bm{x}) - \text{Drop}(\bm{y}) \|_2
= \left( \sum_{i \in [d]} |\text{Drop}(\bm{x})_i - \text{Drop}(\bm{y})_i |^2\right)^{\frac{1}{2}}
\leq \left( \sum_{i \in [d]} \left(\frac{1}{1-p}\right)^2|x_i - y_i|^2\right)^{\frac{1}{2}}
= \frac{1}{1-p} \| \bm{x} - \bm{y} \|_2.
\end{align*}
This completes the proof.
\end{proof}

\subsection{Convolution}\label{sec:conv}
\begin{lem}\label{lem:conv}
Let $\text{Conv} \colon \mathbb{R}^d \to \mathbb{R}^d$ be the convolution function, i.e., $\text{Conv}(\bm{z}) := W\bm{z} + \bm{b}$, where $W$ is a parameter matrix from the convolutional kernel and $\bm{b} \in \mathbb{R}^d$ is a bias term. Then, $\text{Conv}$ is a $\| W \|_2$-Lipschitz function.
\end{lem}
\begin{proof}
From the definition of the operation $\text{Conv}$, for all $\bm{z}_1, \bm{z}_2 \in \mathbb{R}^d$, we have
\begin{align*}
    \| \text{Conv}(\bm{z}_1) - \text{Conv}(\bm{z}_2) \|_2 
    &= \| (W\bm{z}_1 + \bm{b}) - (W\bm{z}_2 + \bm{b}) \|_2 \\
    &= \| W(\bm{z}_1 - \bm{z}_2) \|_2 \\
    &\leq \| W \|_2 \| \bm{z}_1 - \bm{z}_2 \|_2
\end{align*}
This completes the proof.
\end{proof}

\subsection{Upsample}\label{sec:upsample}
\begin{lem}\label{lem:up}
Let $\text{Up} \colon \mathbb{R}^{H \times W} \to \mathbb{R}^{sH\times tW}$ be the upsample operation, where $s,t \in \mathbb{N}$ are the vertical and horizontal scale factors, respectively. Then, $\text{Up}$ is a $\sqrt{st}$-Lipschitz function.
\end{lem}
\begin{proof}
When input $\bm{x} \in \mathbb{R}^{H \times W}$ is upscaled using nearest-neighbor interpolation to construct $\text{Up}(\bm{x}) \in \mathbb{R}^{sH \times t W}$, each element of $\bm{x}$ appears $s \times t$ times in $\text{Up}(\bm{x})$. Therefore, we have
\begin{align*}
\| \text{Up}(\bm{x}) - \text{Up}(\bm{y}) \|_2
=\sqrt{st \sum_{i \in [H\times W]}(x_i - y_i)^2}
=\sqrt{st} \|\bm{x} - \bm{y} \|_2.
\end{align*}
This completes the proof.
\end{proof}

\subsection{Backward Pass}\label{sec:backward}
\begin{lem}\label{lem:back}
Let $T \colon \mathcal{Z} \to \mathcal{Z}$ be the mapping appearing in the backward pass, i.e., $T(\bm{\mathrm{x}}) := \bm{\mathrm{x}} J_{f_{\bm{\theta}}}(\bm{z}^\star) + \frac{\partial l}{\partial \bm{z}^\star}$. Then, $T$ is a $\| J_{f_{\bm{\theta}}}(\bm{z}^\star) \|_2$-Lipschitz mapping.
\end{lem}
\begin{proof}
From the definition of $T$ and $\| \cdot \|_\mathcal{Z}$, we have
\begin{align*}
\| T(\bm{\mathrm{x}}) - T(\bm{\mathrm{y}}) \|_\mathcal{Z} 
&= \left\| (\bm{\mathrm{x}}-\bm{\mathrm{y}}) J_{f_{\bm{\theta}}}(\bm{z}^\star) \right\|_\mathcal{Z} \\
&= \left\| J_{f_{\bm{\theta}}}(\bm{z}^\star)^\top (\bm{\mathrm{x}}-\bm{\mathrm{y}})^\top \right\|_\mathcal{Z} \\
&\leq \left\| J_{f_{\bm{\theta}}}(\bm{z}^\star)^\top \right\|_2 \left\| (\bm{\mathrm{x}}-\bm{\mathrm{y}})^\top \right\|_\mathcal{Z} \\
&\leq \left\| J_{f_{\bm{\theta}}}(\bm{z}^\star) \right\|_2 \left\| \bm{\mathrm{x}}-\bm{\mathrm{y}} \right\|_\mathcal{Z}.
\end{align*}
This completes the proof.
\end{proof}

\section{Additional Experimental Results}
\subsection{Experimental Settings and Hyperparameters}\label{sec:appexp}
In our experiments, all models were trained on single NVIDIA GeForce RTX 4090 GPU and Intel Core i9 13900KF CPU.
\UPDATE{The code is available at \url{https://github.com/iiduka-researches/Lipschitz_mdeq}.
All CIFAR-10 \cite{Krizhevsky2009Lea} experiments in the main text were run for 200 epochs. A batch size of 128 was used except for Jacobian regularization, where a batch size of 96 was employed. Adam \cite{Kingma2015Ada} was used to optimize the empirical loss, with an initial learning rate of 0.001. Other hyperparameters followed prior work, and the configuration is publicly available at \url{https://github.com/iiduka-researches/Lipschitz_mdeq}. For all experiments, the Anderson Acceleration fixed-point solver was used for both the forward and backward passes. The hyperparameters for the phantom gradient were set as 5 iterations and a decay rate of 0.5.}

\subsection{Comparison of Fixed Point Convergence}\label{sec:exp_convergence}
\begin{figure}[htbp]
\begin{minipage}[t]{1\linewidth}
\centering%
\includegraphics[width=0.7\linewidth]{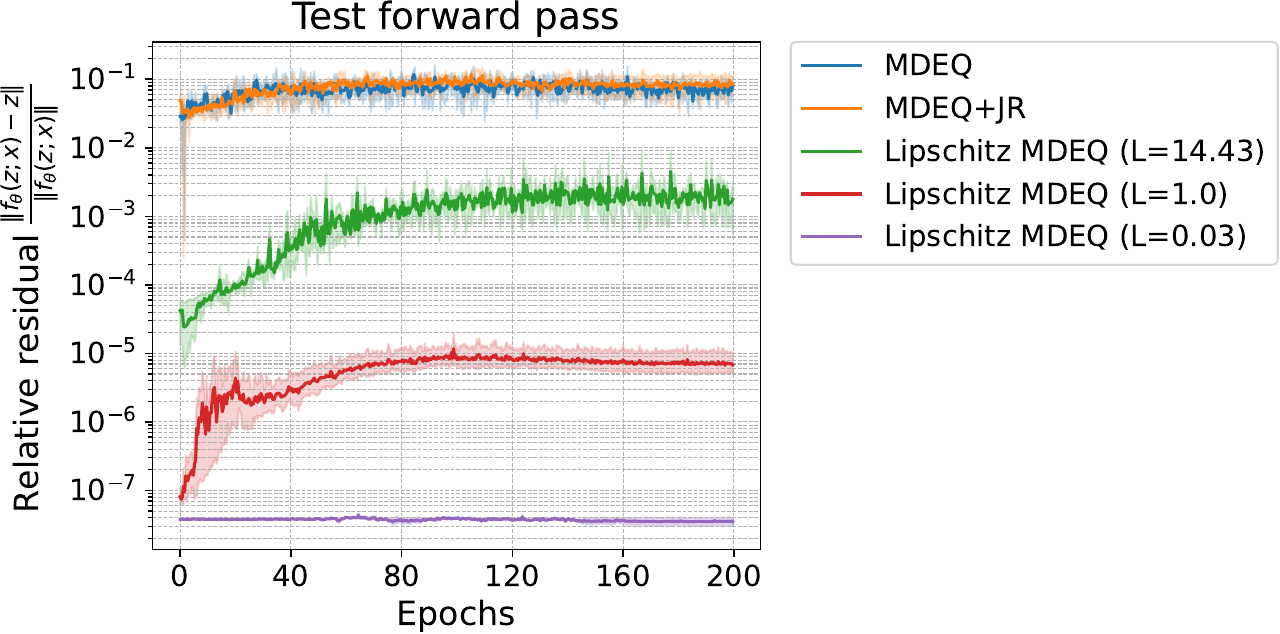}%
\end{minipage}%
\caption{Comparison of fixed-point convergence in the forward pass of testing.}
\label{fig:test_convergence}
\end{figure}

\UPDATE{
\subsection{Experiments on ImageNet}\label{sec:exp_image}
We performed the same experiments on the ImageNet dataset as described in Section \ref{sec:exp}.
In all experiments, we control the Lipschitz constant $L$ by fixing $L_{\text{MGN}}=1.0$, $L_{\text{Conv}^\star}=2.0$, $\alpha_1 = 0.5$, $\alpha_2=0.3$, $n=4$, and $p=0$ and varying $L_{\text{SReLU}}=\{ 0.1, 0.4 \}$. Specifically, $L=\{0.026, 0.794\}$ are obtained when $L_{\text{SReLU}}=\{0.1, 0.4\}$, respectively. All fixed-point iteration evaluation metrics used relative residual, with a threshold of 1e-3. The maximum iteration was set to 14 for both the forward and backward passes. 

\begin{table*}[ht]
    \centering
    \caption{\UPDATE{Comparison of different acceleration techniques for MDEQ on ImageNet. We report top-1 and top-5 accuracy, total training speed, and average number of fixed-point iterations (NFEs) and runtime (Time) for the forward and backward passes. Best results in each column are highlighted in \textbf{bold}.}}
    \label{tab:speedup2}
    \resizebox{\textwidth}{!}{%
    \begin{tabular}{lcccccccccc}
        \toprule
        \textbf{Model and method} & \textbf{Size} & \textbf{\UPDATE{Speed-up}} & \textbf{top-1 Acc.} & \textbf{top-5 Acc.} & NFEs (Train Fwd) & NFEs (Train Bwd)& NFEs (Test Fwd) \\
        \midrule
        MDEQ & 18M &1.0$\times$ & 61.83 & 83.37 & 14.0 & 14.0 & 14.0 \\
        Lipschitz MDEQ ($L=0.794$) &18M & 1.15$\times$ & 62.86 & 84.13 & 13.3 & 12.9 & 13.3 \\
        Lipschitz MDEQ ($L=0.026$) &18M & 1.42$\times$ & 60.89 & 82.74 & \textbf{2.0} & \textbf{2.0} & \textbf{2.0} \\
        Lipschitz MDEQ ($L=0.794$) - Conv$^\star$ &18M & 1.34$\times$ & \textbf{64.21} & \textbf{85.01} & 7.6 & 4.4 & 7.7 \\
        Lipschitz MDEQ ($L=0.026$) - Conv$^\star$ &18M & \textbf{1.48}$\times$ & 62.95 & 84.02 & \textbf{2.0} & \textbf{2.0} & \textbf{2.0}
    \end{tabular}}
    \vspace*{-10pt}
\end{table*}

In both the forward pass during training and testing and the backward pass during training, Lipschitz MDEQ achieves significantly faster fixed-point convergence than existing methods. However, despite a reduction in the number of fixed-point iterations, the speed-up effect is more modest compared to the case of CIFAR10. This is because we set the maximum number of fixed-point iterations to 14 across all settings to reduce computational cost. Typically, experiments using the ImageNet dataset set the maximum number of iterations to 22--30.
Therefore, it is expected that MDEQ's test accuracy could be even higher. Simultaneously, since increasing the number of iterations does not cause MDEQ's fixed-point iterations to converge, it is anticipated that the acceleration effect of our model will be further emphasized.

The results of the ablation study in Section \ref{sec:ablation} indicate that the setting yielding the greatest experimental acceleration effect is reverting the Lipschitz MDEQ's Conv$^\star$ operation back to the Conv operation. Therefore, Table \ref{tab:speedup2} also reports the results when the Conv$^\star$ operation is removed.
Even in the case of ImageNet, removing Conv$^\star$ modifications proved effective, leaving the challenge of better modifications to the convolutional operations.
}

\end{document}